\documentclass[11pt,letterpaper]{article}

\usepackage{hyperref}
\hypersetup{
     colorlinks   = true,
     linkcolor = black,
     urlcolor    = teal,
	 citecolor = teal
}
\usepackage[margin=1in]{geometry}
\usepackage{amssymb,amsthm,amsmath,amssymb,wrapfig,dsfont,authblk,mathrsfs}
\usepackage{times}
\usepackage[dvipsnames]{xcolor}
\usepackage{enumitem}
\usepackage{amsfonts}
\usepackage[langfont=typewriter,funcfont=slant]{complexity}
\usepackage{framed}
\usepackage{soul}
\usepackage{thmtools,thm-restate}
\usepackage{csquotes}
\usepackage[nameinlink, noabbrev, capitalize]{cleveref}
\usepackage{microtype}
\usepackage{placeins}
\usepackage{natbib}

\usepackage{booktabs} %
\usepackage[ruled]{algorithm2e} %

\SetAlFnt{\small}
\SetAlCapFnt{\small}
\SetAlCapNameFnt{\small}
\SetAlCapHSkip{0pt}
\IncMargin{-\parindent}

\setcitestyle{authoryear}

\usepackage{graphicx}
\usepackage{multirow,hhline,tabularx,xcolor,wrapfig,tikz}
\usepackage{thm-restate}

\renewcommand{\vec}[1]{\mathbf{#1}}

\usepackage{soul}

\newcommand{\hinge}{\text{hinge}}
\SetCommentSty{mycommfont}

\newcommand{\cX}{\mathcal{X}}

\newcommand{\cY}{\mathcal{Y}}
\newcommand{\cH}{\mathcal{H}}

\newcommand{\cD}{\mathcal{D}}

\newcommand{\cE}{\mathcal{E}}

\newcommand{\regret}{\mathsf{Regret}}
\newcommand{\mistake}{\mathsf{Mistake}}
\newcommand{\transpose}{\mathsf{T}}

\newcommand{\OPT}{\mathsf{OPT}}

\newcommand{\indicator}[1]{\mathbf{1}\!\left\{#1\right\}}

\newcommand{\w}{\mathbf{w}}
\newcommand{\z}{\mathbf{z}}
\newcommand{\x}{\mathbf{x}}

\newcommand{\cF}{\mathcal{F}}

\newcommand{\BR}{\mathsf{BR}}

\newcommand{\Value}{\mathsf{Value}}
\newcommand{\Cost}{\mathsf{Cost}}
\newcommand{\Dist}{\mathsf{dist}}
\newcommand{\Halving}{\mathsf{Halving}}

\newtheorem{theorem}{Theorem}[section]

\newtheorem{proposition}[theorem]{Proposition}

\newtheorem{remark}[theorem]{Remark}
\newtheorem{example}[theorem]{Example}

\crefname{definition}{definition}{definitions}
\Crefname{definition}{Definition}{Definitions}
\crefname{prop}{proposition}{propositions}
\Crefname{Prop}{Proposition}{Propositions}
\crefname{lemma}{Lemma}{Lemmas}
\crefname{section}{Section}{Sections}
\crefname{subsubsubsection}{Section}{Sections}
\crefname{remark}{Remark}{Remarks}
\crefname{figure}{Figure}{Figures}
\crefname{table}{Table}{Tables}
\Crefname{lemma}{Lemma}{Lemmas}
\crefname{theorem}{Theorem}{Theorems}
\Crefname{theorem}{Theorem}{Theorems}
\crefname{algo}{Algorithm}{Algorithms}

\begin{document}
\title{Fundamental Bounds on Online Strategic Classification\footnote{Authors are ordered alphabetically.}}
\author{Saba Ahmadi$^\dagger$}
\author{Avrim Blum$^\dagger$}
\author{Kunhe Yang$^\ddagger$}
\affil{$^\dagger$Toyota Technological Institute at Chicago\\
$^\ddagger$University of California, Berkeley\\
{\small\texttt{\{saba,avrim\}@ttic.edu\quad kunheyang@berkeley.edu}}}

\date{\today}

\maketitle

\allowdisplaybreaks
\begin{abstract}
We study the problem of online binary classification where strategic agents can manipulate their observable features in predefined ways, modeled by a manipulation graph, in order to receive a positive classification.   We show this setting differs in fundamental ways from classic (non-strategic) online classification.  For instance, whereas in the non-strategic case, a mistake bound of $\ln\!|\mathcal{H}|$  is achievable via the halving algorithm when the target function belongs to a known class $\mathcal{H}$, we show that no deterministic algorithm can achieve a mistake bound $o(\Delta)$   in the strategic setting, where $\Delta$ is the maximum degree of the manipulation graph (even when $|\mathcal{H}|=\mathcal{O}(\Delta)$).  We complement this with a general algorithm achieving mistake bound $\mathcal{O}(\Delta\ln|\mathcal{H}|)$.
 We also extend this to the \emph{agnostic} setting,
 and show that this algorithm achieves a $\Delta$ multiplicative regret (mistake bound of $\mathcal{O}(\Delta\cdot {\rm OPT} + \Delta\cdot\ln|\mathcal{H}|)$), and that no deterministic algorithm can achieve $o(\Delta)$ multiplicative regret.  
 
Next, we study two randomized models based on whether the random choices are made before or after agents respond, and show they exhibit fundamental differences.
In the first, \emph{fractional} model, at each round the learner deterministically chooses a probability distribution over classifiers inducing expected values on each vertex (probabilities of being classified as positive), which the strategic agents respond to. We show that any learner in this model has to suffer linear regret.  On the other hand, in the second
\emph{randomized algorithms} model, while the adversary who selects the next agent must respond to the learner's probability distribution over classifiers, the agent then responds to the actual hypothesis classifier drawn from this distribution.  Surprisingly, we show this model is more advantageous to the learner, and we design randomized algorithms that achieve sublinear regret bounds against both oblivious and adaptive adversaries. 
\end{abstract}

\section{Introduction}

\emph{Strategic classification} concerns the problem of learning classifiers that are robust to gaming by self-interested agents~\cite{10.1145/2020408.2020495,Hardt2016}. An example is deciding who should be qualified for getting a loan and who should be rejected. Since applicants would like to be approved for getting a loan, they may spend efforts on activities that do not truly change their underlying loan-worthiness but may cause the classifier to label them as positive. An example of such efforts is holding multiple credit cards. Such gaming behaviors have nothing to do with their true qualification but could increase their credit score and therefore their chance of getting a loan. 
 Strategic classification is particularly challenging in the \emph{online} setting where data points arrive in an online manner. In this scenario, the way that examples manipulate depends on the \emph{current classifier}. %
Therefore, the examples' behavior changes over time and it may be different from examples with similar features observed in the previous rounds. Additionally, there is no useful source of unmanipulated data since there is no assumption that the unmanipulated data is coming from an underlying distribution.

Strategic agents are modeled as having bounded manipulation power, and a goal of receiving a positive classification. The set of plausible manipulations has been characterized in two different ways in the literature. The first model considers a geometric setting where each example is a point in the space that can move in a ball of bounded radius (e.g., %
~\citet{dong2018strategic,chen2020learning,haghtalab-ijcai2020,ahmadi2021strategic,ghalme2021strategic}). Another model is an abstraction of feasible manipulations using a \emph{manipulation graph} that first was introduced by \citet{zhang2021incentive}. %
We follow the second model and formulate possible manipulations using a graph. 
Each possible feature vector is modeled as a node in this graph, and an edge from $\vec{x}\rightarrow \vec{x'}$ in the manipulation graph implies that an agent with feature vector $\vec{x}$ may modify their features to $\vec{x'}$ if it helps them to receive a positive classification. We consider the problem of online strategic classification given an underlying manipulation graph. Our goal is to minimize the \emph{Stackelberg regret} which is the difference between the learner's cumulative loss and the cumulative loss of the best-fixed hypothesis against the same sequence of agents, but best-responding to this fixed hypothesis. 

In this paper, we consider three models with different levels of randomization. First, we consider the scenario where the learner can pick \emph{deterministic} classifiers. A well-known deterministic algorithm in the context of online learning is the \emph{halving} algorithm, which classically makes at most $O(\ln|\cH|)$ mistakes when the target function belongs to class $\cH$. First, we show that when agents are strategic, the \emph{halving} algorithm fails completely and may end up making mistakes at every round even in this realizable case. Moreover, we show that no deterministic algorithm can achieve a mistake bound $o(\Delta)$ in the strategic setting, where $\Delta$ is the maximum degree of the manipulation graph, even when $|\cH|=O(\Delta)$.  We complement this result with a general algorithm achieving mistake bound $O(\Delta\ln|\mathcal{H}|)$ in the strategic setting.
We further extend this algorithm to achieve $O(\Delta)$ multiplicative regret bounds in the non-realizable (agnostic) strategic setting, giving matching lower bounds as well.

Our next model is a {\em fractional} model where  at each round the learner chooses a probability distribution over classifiers, inducing expected values on each vertex (the probability of each vertex being classified as positive), which the strategic agents respond to.
The agents' goal is to maximize their utility by reaching a state that maximizes their chance of getting classified as positive minus their modification cost. For this model, we show regret upper and lower bounds similar to the deterministic case. 

In the last model, the learner again picks a probability distribution over classifiers, but now, while the adversary who selects the next agent must respond to this probability distribution, the agent responds to the actual classifier drawn from this distribution. That is, in this model, the random draw occurs after the adversary's selection of the agent but before the agent responds, whereas in the fractional model the random draw occurs after the agent responds. Surprisingly, we show this model is %
not only more transparent to the agents, but also
more advantageous to the learner than the fractional model. 
{ We argue that transparency can make the learner and agents cooperate against the adversary in a way that would be more beneficial to both parties, which is an interesting phenomenon that differentiates the strategic setting from the nonstrategic one.} %
{In this model, }we design randomized algorithms that achieve sublinear regret bounds against both oblivious and adaptive adversaries. 
We give a detailed overview of our results in~\Cref{sec:overview-results}.

\subsection{Related Work}

Our work builds upon a growing line of research, initiated by \citet{10.1145/1014052.1014066,dekel2008incentive,10.1145/2020408.2020495}, that studies learning from data provided by strategic agents.~\citet{Hardt2016} differentiated the field of strategic classification from the more general area of learning under adversarial perturbations;
they introduced the problem of \emph{strategic classification} and modeled it as a sequential game between a jury that deploys a classifier and an agent that best responds to the classifier by modifying their features at a cost.

Following the framework of \citet{Hardt2016}, recent works have focused on both the offline setting where examples come from underlying distributions~\citep{zhang2021incentive,sundaram2021pac,lechner2022learning,pmlr-v119-perdomo20a} and the online settings where examples are chosen by an adversary in a sequential manner~\citep{dong2018strategic,chen2020learning,ahmadi2021strategic}.
\citet{milli-etal,10.1145/3287560.3287597} extend the setting considered by \citet{Hardt2016} to the case that heterogeneous sub-populations of strategic agents have different manipulation costs and studied other objectives such as social burden and fairness. %
A number of other works focus on incentivizing agents to take improvement actions that increase their true qualification as opposed to gaming actions~\citep{10.1145/3417742,Alon2020MultiagentEM,haghtalab-ijcai2020,ahmadi_et_al:LIPIcs.FORC.2022.3,bechavod2022information}.
The works by \citet{pmlr-v119-shavit20a,pmlr-v119-perdomo20a,bechavod2021gaming} study causal relationships between observable features and outcomes in strategic classification.~\citet{levanon2021strategic} provide a practical learning framework for strategic classification%
. 
Recent works relax the assumption that strategic agents best respond to the classifiers and consider alternative assumptions such as noisy response~\citep{jagadeesan2021alternative}, learning agents~\citep{zrnic2021leads}, and non-myopic agents~\citep{haghtalab2022learning}.

Our work is most closely related to that of~\citet{zhang2021incentive,lechner2022learning}, which captures the set of plausible manipulations using an underlying \emph{manipulation graph}, where each edge $\vec{x}\rightarrow \vec{x'}$ represents a plausible manipulation from features $\vec{x}$ to $\vec{x}'$. This formulation is in contrast to a geometric model where agents' features are vectors in a $d$-dimensional space, with manipulation cost captured by some distance metric. As a result, agents in the geometric setting move in a ball of bounded radius~\citep{dong2018strategic,chen2020learning,haghtalab-ijcai2020,ahmadi2021strategic,ghalme2021strategic,sundaram2021pac}. However, the work of \citet{zhang2021incentive,lechner2022learning} focuses on the \emph{offline} PAC learning setting. Our work can be considered as generalizations of their model to the \emph{online learning} setting.

Our work is also connected to the line of work that studies randomness and transparency in strategic classification. In terms of \emph{classification accuracy} in the offline setting, \citet{Braverman2020TheRO} shows that in a one-dimensional feature space, both committed randomness (probabilistic classifiers) and noisy features under deterministic classifiers can improve accuracy, and the optimal randomized classifier has a structure where agents are better off not manipulating. On the other hand, \citet{ghalme2021strategic} gives sufficient conditions under which \emph{transparency} is the recommended policy for improving predictive accuracy. Our paper combines the insights of both papers in the online setting, where we show that randomness is necessary against the adversary that selects agents, but transparency is more advantageous when it comes to the strategic agents themselves (see \Cref{sec:discussion-transparency} for more discussions). In addition to accuracy, there are also studies about the societal implications of randomization and transparency in the presence of multiple subpopulations, such as information discrepancy~\citep{bechavod2022information} and fairness~\citep{immorlica2019access,kannan2019downstream,frankel2022improving,Braverman2020TheRO}.

\section{Model}

\subsection{Strategic Classification}
Let $\cX$ denote the space of agents' features, and $\cY=\{+1,-1\}$ denote the space of labels. %
We consider the task of sequential classification where the learner aims to classify a sequence of agents $\{u_t,y_t\}_{t=1}^T$ that arrive in an online fashion. Here, we assume $u_t\in\cX$ is the true feature set of agent $t$ and $y_t\in\cY$ is the true label. We call agents with $y_t=+1$ \emph{true positives}, and the ones with $y_t=-1$ \emph{true negatives}. %
Importantly, we make minimum assumptions on the sequence of agents, and our results apply to the case of adversarially chosen sequences.
A hypothesis $h:\cX\rightarrow \cY$ (also called a {\em classifier} or an {\em expert}) is a function that assigns labels to the agents $u\in\cX$. 
Given a hypothesis $\mathcal{H}:\cX\to\cY$, our goal is to bound the total number of mistakes made by the learner, compared to the best classifier $h^\star\in\cH$ in hindsight. 

To model the gaming behavior in real-life classification tasks, we work with the setting of \emph{strategic classification}, in which agents have the ability to modify their features at a given cost and reach a new observable state.
Formally, strategic classification can be described as a repeated Stackelberg game between the learner (leader) and the agents (followers).
At each step $t\in[T]$, the learner first publicly commits to a classifier $h_t$. Then, the $t$-th agent $(u_t,y_t)$ arrives and responds to $h_t$ by modifying their features from $u_t$ to $v_t$. As a result of manipulation, only $v_t$ (instead of $u_t$) is observable to the learner. 

We assume that $v_t$ is chosen as a best-response ($\BR$) to the announced rule $h_t$, such that the agent's utility is maximized:
\begin{align}
    v_t\in\BR_{h_t}(u_t)\triangleq\arg\max_{v\in \cX} \Big[\Value(h_t(v))-\Cost(u_t,v)\Big].
\end{align}
Here, $\Value(h_t(v))$ indicates the value of outcome $h_t(v)$, which is a binary function that takes the value of $1$ when $h_t(v)=+1$, and $0$ when $h_t(v)=-1$.
In~\Cref{sec:fractional-model}, we consider the generalization of agents best responding to a probability distribution over classifiers, where $\Value(h_t(v))$ becomes the induced expectation on node $v$, i.e., the probability of $v$ getting classified as positive by $h_t$. Equivalently, we refer to $h_t$ as a \emph{fractional classifier} and the induced probabilities as \emph{fractions}.
$\Cost(u_t,v)$ is a known, non-negative cost function that captures the cost of modifying features from $u_t$ to $v$. It is natural to assume $\Cost(u,u)=0$ for all $u\in \cX$. We use $v_t \in \BR_{h_t}(u_t)$ to show the best-response of agent $u_t$ at time-step $t$. Ties are broken by always preferring features with higher $\Value(h_t(\cdot))$, and 
preferring to stay put, i.e. $u_t=v_t$, if $u_t$ is among the set of best-responses that achieves the highest value.

\textbf{Learner's Objective:}
The learner's loss is defined as the misclassification error on the observable %
state: $\ell(h_t,v_t,y_t)=\indicator{y_t\neq h_t(v_t)}$. Since $v_t \in \BR_{h_t}(u_t)$ and has the highest value of $h_t(\cdot)$ according to the tie breaking rule, we also abuse the notation and write $\ell(h_t,\BR_{h_t}(u_t),y_t)=\indicator{y_t\neq\Big. \max\left\{ h_t(v):\ {v\in \BR_{h_t}(u_t)}\right\}}$.
The learner's goal is to minimize the Stackelberg regret with respect to the best hypothesis $h^\star\in\cH$ in hindsight, had the agents best responded to $h^\star$:
\begin{align}
    \regret(T)\triangleq\sum_{t=1}^T \ell(h_t,\BR_{h_t}(u_t),y_t)-\min_{h^\star\in\cH}\sum_{t=1}^T \ell(h^\star,\BR_{h^\star}(u_t),y_t).
\end{align}
For notational convenience, we use $\OPT$ to denote the optimal loss achieved by the best hypothesis:
\begin{align}
    \OPT\triangleq\min_{h^\star\in\cH}\sum_{t=1}^T \ell(h^\star,\BR_{h^\star}(u_t),y_t).
\end{align}
When $\OPT=0$, we call the sequence of agents \emph{realizable}, meaning that there exists a perfect classifier that never makes a mistake had agents best responded to it. Otherwise when $\OPT>0$, we call it \emph{unrealizable} or \emph{agnostic}.

\subsection{Manipluation Graph}
We use graph $G(\cX,\cE)$ to characterize the set of \emph{plausible manipulations}. In graph $G$, each node in $\cX$ corresponds to a state (i.e., features), and
each edge $e=(u,v)\in\cE$ captures a plausible manipulation from $u$ to $v$. 
The cost function $\Cost(u,v)$ is defined as the sum of costs on the shortest path from $u$ to $v$, or $+\infty$ if such a path does not exist.
We present our results for the case of undirected manipulation graphs and show how they can be extended to the case of directed graphs (\Cref{sec:directed-graphs}).

To model the cost of each edge, we consider \emph{weighted graphs} in which each edge $e\in \cE$ is associated with a nonnegative weight $w(e)\in[0,1]$. 
As a special case of the weighted graphs, we also consider \emph{unweighted graphs}, in which each edge takes unit cost, i.e., $w(e)=1$. We remark that in unweighted graphs, agents will move for at most one hop because manipulating the features can increase the value of classification outcomes by at most $1$. To be specific, let $N[u]$ denote the closed neighborhood of state $u\in\cX$, then agent $u$ respond to classifier $h$ as follows:  
if $h(u)$ is negative and there exists a neighbor $v\in N[u]$ with positive $h(v)$, then $u$ will move to $v$; otherwise, $u$ does not move. As a result, the loss function in unweighted graphs can be equivalently expressed as
\[
\ell(h,\BR_h(u),y)=
\begin{cases}
1 &\quad y=+1 \text{, } \forall v\in N[u]: h(v)=-1;\\
1 &\quad y=-1 \text{, } \exists v\in N[u]: h(v)=+1;\\
0 &\quad \text{otherwise}.
\end{cases}
\label{def-loss}
\]

When fractional classifiers are used, we also consider the \emph{free-edge} manipulation model. In this model, we restrict the agent to only moving one hop, where the cost of moving is zero. 
Specifically, each pair of nodes $(u,v)\in\cX^2$ has zero manipulation cost if $(u,v)\in\cE$, otherwise the cost is infinity. 
When agents best respond to classifier $h$ under this cost function, they will move to a one-hop neighbor of their initial state that has the highest probability of getting classified as positive.

\section{Overview of Results}
\label{sec:overview-results}

\begin{table}
    \centering
    \footnotesize
    \begin{tabular}{|c|c|c|c|}
    \hline
    \parbox[c][][c]{2cm}{Type of \\Randomness} &\parbox[c][][c]{2cm}{Manipulation\\ Graph} & {Upper Bound} &Lower Bound\\\hhline{====}
    \parbox[c][][c]{1.9cm}{Deterministic}&
    {Unweighted}
    &
    \begin{tabular}{l|l}
        \multirow{2}{*}{Realizable} & $O(\Delta\ln|\cH|)$\\
        &\Cref{alg:halving} (\Cref{thm:baseline-realizable-upper-bound})\\\hline
        \multirow{2}{*}{Agnostic} & $O(\Delta\cdot\OPT+\Delta\ln|\cH|)$
        \\
        & \Cref{alg:biased-weighted-maj-vote} (\Cref{thm:biased-weighted-maj-vote-mistake-bound})
    \end{tabular}
    &
    \begin{tabular}{l|c}
        \multirow{2}{*}{Realizable} & $\Delta-1$\\
        &\Cref{thm:deterministic-lower-bound}\\\hline
        \multirow{2}{*}{Agnostic} & $\Delta\cdot\OPT$
        \\
        & \Cref{thm:deterministic-lower-bound}
    \end{tabular}
    \\\hhline{====}
    \multirow{2}{*}{\parbox[c][][c]{1.9cm}{Fractional \\Classifiers \\{\tiny (random choice after agents respond)}}}&Free-edges&
    \parbox[c][1cm][c]{4cm}{$O(\Delta\cdot\OPT+\Delta\ln|\cH|)$ \\
    \Cref{alg:biased-weighted-maj-vote} (\Cref{thm:biased-weighted-maj-vote-mistake-bound})}
    &
    \parbox[c][1cm][c]{3cm}{
        $\frac{\Delta}{2}\cdot\OPT$\\
        \Cref{thm:fractional-onehop-lower-bound}
    }\\\hhline{~---}
    &Weighted&
    \parbox[c][1cm][c]{4cm}{$O(\tilde{\Delta}\cdot\OPT+\tilde{\Delta}\ln|\cH|)$ \\
    \Cref{prop:fractional-multi-hop-upper-bound}}
    &
    \parbox[l][1cm][c]{3cm}{
        {$\frac{{\Delta}}{4}\cdot\OPT\ \left(\frac{\Tilde{\Delta}}{4}\cdot\OPT\right)$ }\\
        \Cref{thm:fractional-multi-hops-lower-bound}
    }\\\hhline{====}
    \parbox[c][][c]{1.9cm}{Randomized\\ Algorithms
    \\{\tiny (random choice before agents respond)}}&{
    Unweighted}&
    \begin{tabular}{l|l}
        \multirow{2}{*}{Oblivious} & $O\left(T^{\frac{2}{3}}\ln^{\frac{1}{3}}|\mathcal{H}|\right)$\\
        &\Cref{alg:reduction-MAB-FIB} (\Cref{thm:regret-alg-oblivious})\\\hline
        \multirow{4}{*}{Adaptive} & $\widetilde{O}\left(T^{\frac{3}{4}}\ln^{\frac{1}{4}}|\mathcal{H}|\right)$
        \\
        & \Cref{alg:reduction-adaptive} (\Cref{thm:regret-alg-adaptive-reduction})\\\hhline{~-}
        & $\widetilde{O}\left(\sqrt{T|\cH|\ln|\mathcal{H}|}\right)$
        \\
        & Vanilla EXP3 Algorithm
    \end{tabular}
    &Open\\\hline
    \end{tabular}
    \caption{\small{This table summarizes the main results of this paper for the model of deterministic classifiers, fractional classifiers, and randomized algorithms respectively. We use $\Delta$ to denote the maximum degree of the manipulation graph, and $\tilde{\Delta}$ to denote the maximum degree of the expanded manipulation graph, which is constructed from a weighted graph by connecting all %
    pairs of nodes $(u,v)\in \cX^2$ where $\Cost(u,v)\leq 1$.
    Although the table is presented in terms of undirected graphs, we remark that all the upper and lower can be extended to the setting of directed graphs, with the degrees to be replaced by the corresponding out-degrees, %
    see \Cref{sec:directed-graphs} for an example in the setting of deterministic classifiers}.
    $\OPT$ stands for the optimal number of mistakes
    achieved by the best hypothesis in class $\cH$.
    }
      \label{table-of-results}
\end{table}

Our work considers three types of randomness: deterministic, fractional classifiers, and randomized algorithms. In the \emph{deterministic} model, the learner is constrained to using deterministic algorithms to output a sequence of deterministic classifiers. In the \emph{fractional classifiers} model, the learner is allowed to output a probability distribution over classifiers at every round, inducing fractions on every node that represent their probability of being classified as positive. The agents best respond to these fractions before the random labels are realized. In the last \emph{randomized algorithms} model, the learner outputs a probability distribution over classifiers as in the fractional model, and the adversary may pick $u_t$ based on these probabilities, but now the agents respond to the true realized classifier in selecting $v_t$. We summarize our main results 
in \Cref{table-of-results}.

\textbf{Deterministic Classifiers:} 
In the case of deterministic classifiers, we model strategic manipulations by unweighted graphs that have unit cost on all edges. We first consider the realizable setting where the perfect classifier lies in a finite hypothesis class $\cH$, and show fundamental differences between the non-strategic and strategic settings.
In the non-strategic setting, the deterministic algorithm $\Halving$ achieves $O(\ln|\cH|)$ mistake bound. However, in the strategic setting, we show in \Cref{example:halving-fails} that the same algorithm can suffer from an infinite number of mistakes. 

In \Cref{sec:deterministic}, we analyze the strategic setting and provide upper and lower bounds of the mistake bound, both characterized by the \emph{maximum degree} of vertices in the manipulation graph, which we denote with $\Delta$.
On the lower bound side, we show in \Cref{thm:deterministic-lower-bound} that no deterministic algorithm is able to achieve $o(\Delta)$ mistake bound, and this barrier exists even when $|\cH|=O(\Delta)$.
On the upper bound side, we propose \Cref{alg:halving} that achieves mistake bound $O(\Delta\ln|\cH|)$ by incorporating the graph structure into the vanilla $\Halving$ algorithm.

We then move to the agnostic strategic setting and propose \Cref{alg:biased-weighted-maj-vote} which achieves a mistake bound of $O(\Delta\cdot\OPT+\Delta\ln|\mathcal{H}|)$, where $\OPT$ denotes the minimum number of mistakes made by the best classifier in $\mathcal{H}$. 
This bound is $\Delta$-multiplicative of the bound achieved by the weighted majority vote algorithm in the non-strategic setting. {Furthermore, we extend our results to the setting where the input graph is a supergraph of the true manipulation graph, i.e., it contains all the edges in the true manipulation graph but it might also contain some fake edges (\Cref{remark:supergraph}).}
We complement this result with a lower bound showing that no deterministic algorithm can achieve $o(\Delta\cdot\OPT)$ mistake bound. 
In order to overcome the $\Delta$-multiplicative barrier, we study the use of randomization in the next two models.

\textbf{Fractional Classifiers:} %
In this setting, we consider 
two models of cost function:
the \emph{free-edges} cost model, where traveling one edge is cost-free but the second edge costs infinity, and the \emph{weighted graph} model, where agents can travel multiple edges and pay for the sum of costs of edges.
In the free-edges model, we show that no learner can overcome the mistake lower bound $\frac{\Delta}{2}\cdot\OPT$, and provide an upper bound of $O(\Delta\cdot\OPT+\Delta\ln|\mathcal{H}|)$ based on~\Cref{alg:biased-weighted-maj-vote}.
In the weighted graph model, we show a mistake lower bound of $\frac{{\Delta}}{4}\cdot\OPT$, and an upper bound of $O(\Tilde{\Delta}\cdot\OPT+\Tilde{\Delta}\ln|\mathcal{H}|)$, which is obtained by running \Cref{alg:biased-weighted-maj-vote} on the \emph{expanded manipulation graph} $\Tilde{G}$ that is constructed by connecting all pairs of nodes $(u,v)\in \cX^2$ where $\Cost(u,v)\leq 1$, and $\Tilde{\Delta}$ denotes the maximum degree of $\Tilde{G}$.
In particular, our construction for the lower bound satisfies $\Tilde{\Delta}=\Delta$, so this result also implies that no learner is able to surpass the $\frac{\Tilde{\Delta}}{4}$-multiplicative regret.

Our results in this setting indicate that using fractional classifiers cannot help the learner to achieve $o(\Delta\cdot\OPT)$ regret. 
To resolve this issue, we move on to the randomized algorithms model where the learner realizes the random choices in transparency to the agents.

\textbf{Randomized Algorithms:} 
In this setting, the learner uses randomized algorithms that produce probability distribution over deterministic classifiers at each round.
The key difference from the fractional classifiers setting is, although the adversary still chooses agent $(u_t,y_t)$ based on the distribution, the agent will best respond to the classifier to be used after it is sampled from the distribution.
Surprisingly, we show that revealing the random choices to the agents can make the interaction more fruitful for both the agents and the learner, as the learner is now able to achieve vanishing regret without the multiplicative dependency on $\Delta$ or $\Tilde{\Delta}$. This demonstrates an interesting difference between strategic and non-strategic settings {from the learner's perspective}: whereas delaying the realization of random bits is helpful in non-strategic settings, it is more helpful to realize the random choices \emph{before agents respond} in the strategic setting. We refer the readers to \Cref{sec:discussion-transparency} for more discussions about this difference.

As for algorithms and upper bounds in this setting, we first show that the vanilla EXP3 algorithm on expert set $\cH$ gives us a regret upper bound of $O\left(\sqrt{T|\mathcal{H}|\ln|\mathcal{H}|}\right)$.
To improve the dependency on $|\cH|$, we design two algorithms that simultaneously observe the loss of all experts by using an all-positive classifier at random time steps to stop the manipulations.
In particular, \Cref{alg:reduction-MAB-FIB} achieves regret upper bound of $O\left(T^{\frac{2}{3}}\ln^\frac{1}{3}|\mathcal{H}|\right)$ against oblivious adversaries; and \Cref{alg:reduction-adaptive} achieves regret bound of $\widetilde{O}\left(T^{\frac{3}{4}}\ln^{\frac{1}{4}}|\mathcal{H}|\right)$ against general adaptive adversaries.
We also extend this algorithmic idea to the linear classification setting where original examples are inseparable and obtain an upper bound in terms of the hinge loss of the original data points, resolving an open problem proposed in \citet{ahmadi2021strategic}. Although, our mistake bound has an extra $O(\sqrt{T})$ additive term compared to their bound for the case that original data points are separable.

\textbf{Two Populations:}
We propose an extension to our model in which agents are divided into two populations with heterogeneous manipulation power: group $A$ agents face a cost of 0.5 on each edge, whereas group $B$ agents face a cost of 1. 
We assume that group membership is a protected feature, and is observable only after the classifier is published.
In \Cref{sec:two-populations}, we present an algorithm with a $\min\left\{\Delta+1+\frac{1}{\beta},\ \Delta^2+2\right\}$-multiplicative regret, where $\beta$ is the probability that agents are assigned to group $B$.

\section{Deterministic Classifiers}
\label{sec:deterministic}
\subsection{Realizable Case}
\label{sec:realizable}

In the realizable case, we assume that there exists a perfect expert $h^\star\in\cH$ with zero mistakes, i.e., $\OPT=0$. This implies that for all time steps $t\in [T]$, we have $\ell(h^\star,\BR_{h^\star}(u_t),y_t)=0$. 
In this case, our goal of bounding the Stackelberg regret coincides with the mistake bound:
\begin{align}
    \mistake(T)\triangleq\sum_{t=1}^T \ell(h_t,\BR_{h_t}(u_t),y_t).
\end{align}
For notational convenience, let $S^\star$ denote %
the set of nodes in $\cX$ with positive labels under $h^\star$, namely $S^\star\triangleq\left\{u\in\cX:\ h^\star(u)=+1\right\}$. Then realizability implies that $S^\star$ must satisfy two properties: (1) all the true positives can reach $S^\star$ within no more than one hop;
(2) No true negatives can reach $S^\star$ in one hop.
We formalize these two properties in \Cref{prop:dominating-set}.

\begin{proposition}
    \label{prop:dominating-set}
 In the realizable case, there exists a subset of nodes $S^\star\subseteq \cX$ such that $S^\star$ is a \emph{dominating set} for all the true positives $u_t$, i.e. $\Dist(u_t, S^\star)\leq 1$. Additionally, none of the true negatives $u_t$ are dominated by $S^\star$, i.e. $\Dist(u_t, S^\star)>1$, where $\Dist(u, S^\star)$ represents the minimum distance from node $u$ to the set $S^\star$.

\end{proposition}
\subsubsection{The failure of vanilla Halving}
In the problem of nonstrategic online classification with expert advice, the well-known $\Halving$  algorithm achieves a mistake bound of $\mistake(T)=\mathcal{O}(\ln{|\mathcal{H}|})$. In each iteration, $\Halving$ uses the majority vote of remaining experts to make predictions on the next instance, which ends up reducing the number of remaining experts by at least half on each mistake. Since there are $|\cH|$ experts at the beginning and at least one expert at the end, the total number of mistakes is bounded by $\mathcal{O}(\ln{|\mathcal{H}|})$. However, in the following example, we show that when agents are strategic, the vanilla $\Halving$ algorithm may suffer from an infinite number of mistakes, as do
two extensions of vanilla $\Halving$ that consider the best response function before taking majority votes.
Moreover, our construction indicates that these algorithms fail even when the sequence of agents is chosen by an oblivious adversary.

\begin{figure}
\centering
\begin{tikzpicture}
\def \n {20}
\def \radius {2cm}

\def \margin {8} %

\node[draw, circle] at (360:0mm) (ustar) {$x_0$};
\node at (352:\radius*0.3) {\textcolor{red}{$-$}};
\foreach \i [count=\ni from 0] in {\Delta,1,2,3,4}{
  \node[draw, circle] at ({120-\ni*36}:\radius) (u\ni) {$x_{\i}$};
  \node at ({116-\ni*36}:\radius*1.3) {\textcolor{red}{$-$}};
  \draw[thick] (ustar)--(u\ni);
}
\node[draw, circle] at ({225}:\radius) (ui) {$x_{i}$};
\node at ({225}:\radius*1.3) {\textcolor{red}{+}};
\draw[thick] (ustar)--(ui);
\foreach \i in {5,6,8,9}{
  \node[circle] at ({120-\i*36}:\radius) (aux) {\phantom{$u_{5}$}};
  \draw[dotted, thick, shorten >=1mm, shorten <=2mm] (ustar)--(aux);
}

\draw[dotted, semithick, red] (-40:\radius/2) arc[start angle=-40, end angle=-120, radius=\radius/2];
\draw[dotted, semithick, red] (-150:\radius/2) arc[start angle=-150, end angle=-230, radius=\radius/2];
\end{tikzpicture}
\caption{Expert $h^i$}
\label{fig:lower-bound-deterministic}
\end{figure}
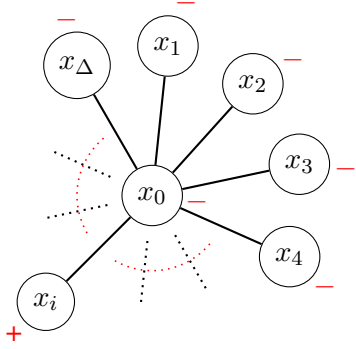

\begin{example}
    \label{example:halving-fails}
Consider the following manipulation graph $G(\cX,\cE)$ and hypothesis class $\cH$: $G(\cX,\cE)$ is a star that includes a central node $x_0$, and $\Delta$ leaves $x_1,\cdots,x_{\Delta}$. Hypothesis class $\mathcal{H}=\{h^1,\cdots,h^{\Delta}\}$, where each $h^i\in \cH$ assigns positive to $x_i$, and negative to all other nodes in $\cX$ (see \Cref{fig:lower-bound-deterministic}). The perfect expert is $h^\star=h^j\in \mathcal{H}$ for some $j\in[\Delta]$ unknown to the learner.
\end{example}

Now consider two algorithms: the vanilla $\Halving$ algorithm and the variant that performs an expansion of positive region on top of $\Halving$.

\begin{enumerate}
    \item \textbf{Vanilla $\Halving$}.
    
    Consider the following sequence of agents: at every time $t$, the same agent with initial position $u_t=x_0$ and label $y_t=+1$ arrives. We claim that the $\Halving$ algorithm makes mistakes on each agent regardless of the total number of rounds executed. First, note that this sequence is realizable with respect to class $\cH$: for all $h^i\in \mathcal{H}$, we have $\BR_{h^i}(x_0)=x^i$ and $h^i(x_i)=+1$, so each $h^i$ classifies $(x_0,+1)$ correctly in isolation. Therefore, any expert in $\mathcal{H} $ achieves %
    zero mistakes on this sequence of agents. 
    
    Now consider the vanilla $\Halving$ algorithm. Initially, for each node $x\in\cX$, there is at most one expert in $\cH$ that labels it as positive. Therefore, the majority vote classifier of $\cH$ labels every node as negative. In response to this all-negative majority vote classifier, the first agent $(x_0,+1)$ stays put and gets classified as negative mistakenly.
    However, %
    we know that each classifier $h^i$ predicts correctly on $(x_0,+1)$. As a result, none of the experts get discarded. Therefore, a mistake is made by the learner, but no progress is made in terms of shrinking the set $\mathcal{H}$. The same agent appears at every round, resulting in the $\Halving$ algorithm %
    making mistakes in each round.

    \item \textbf{A strategic variant of $\Halving$}.

    Now consider a different voting rule for taking the majority-vote classifier based on the best-response function: Let $\overline{h}(u)=h(\BR_{h}(u))$ for all $h\in\cH$ and $u\in\cX$, and suppose that the learner runs $\Halving$ on the hypothesis class $\overline{\cH}=\{\overline{h^1},\cdots,\overline{h^\Delta}\}$. Specifically, for each $h^i\in \mathcal{H}$, $\overline{h^i}(x_0)=h^i(\BR_{h^i}(x_0))=h^i(x_i)=+1$, therefore the majority-vote classifier predicts positive on $x_0$. On the other hand, the majority vote classifier predicts negative on all the leaves.
    Now, suppose the adversary secretly chooses $j\in[\Delta]$ and constructs a sequence in which $h^j$ is realizable as follows: at each time step $t$, selects an example with true label $y_t=-1$ and initial position $u_t=x_i\in \cX\setminus\{x_0,x_j\}$.
    Note that all classifiers in ${\cH}$ except ${h^i}$ will classify $(u_t,y_t)$ correctly.
    However, the majority vote classifier will make a mistake because $u_t$ can manipulate to $x_0$ and get classified as positive. Once the mistake is made, had the true location $u_t=x_i$ been observable, the learner could have shrunk the size of $\cH$ by discarding $h^i$. However, $u_t$ is hidden from the learner, so the learner would not know which classifier is making a mistake. Therefore, it cannot make progress by excluding at least one expert from $\mathcal{H}$ in each round.

   \item \textbf{Another strategic variant of $\Halving$.}
   
   The positive region of $h^{\text{maj}}$ in the previous variation can be reached by all the nodes in the graph, which makes gaming too easy for the agents. Now, suppose the learner's goal is to shrink the positive region of $h^{\text{maj}}$ and get a new classifier $h$, such that the positive region of $h$ can only be reached by the true positives under $h^{\text{maj}}$, but none of the true negatives.
   
   We use the same example as above to show the failure of this algorithm because such $h$ does not exist. Recall that the positive region of $h^{\text{maj}}$ contains only the central node $x_0$. 
   Suppose such an $h$ exists, then $x_0$ cannot belong to the positive region of $h$, because it can be reached by all leaf nodes $x_i$, which are true negatives under $h^{\text{maj}}$. 
   In addition, no leaf nodes should be included in the positive region of $h$ either. This implies that the positive region of $h$ is empty, which contradicts with the assumption that true positive node $x_0$ can reach it.
   For this reason, the learner is unable to find an $h$ satisfying this property. %

\end{enumerate}

\Cref{example:halving-fails} indicates that taking majority votes fails in the strategic setting. One crucial point is that the leaves do not meet the threshold for \emph{majority}, and therefore they are always negative under the majority vote classifier (whether we consider the best response function or not) and thus indistinguishable, weakening the learner's leverage to identify the optimal expert. In fact, in this example, the only evidence for removing an expert is a false positive agent at the corresponding leaf node, so the learner should classify the leaves as positive in order to make progress. Therefore, one needs to lower the threshold for majority votes to increase the likelihood of false positives and make more room for improvement.

In the next section, we propose an algorithm based on the idea of \emph{biased majority vote} in favor of positive predictions, which provably achieves finite mistake bounds against any adversarial sequence of strategic agents. We show that compared to the nonstrategic setting, the extra number of mistakes made by the learner is closely characterized by the maximum degree of the manipulation graph.

\subsubsection{Upper Bound: Biased Majority-Vote Algorithm}
In this section, we propose a biased version of the majority vote algorithm for the realizable strategic setting. The algorithm proceeds in rounds as follows: At each round $t$, a new agent arrives and gets observed as $v_t$. From the remaining set of experts, if at least $1/(\Delta+2)$ fraction of them classify $v_t$ as positive, then the algorithm predicts positive. If the algorithm made a mistake, all the experts that predicted positive get removed from $\mathcal{H}$. 
If less than $1/(\Delta+2)$ fraction of the experts classify $v_t$ as positive, the algorithm predicts negative. If the prediction was wrong, then each expert that labeled all the vertices in the neighborhood of $v_t$, i.e. $N[v_t]$, as negative gets removed from $\mathcal{H}$.
We present this algorithm in \Cref{alg:halving} and analyze its mistake bound in \Cref{thm:baseline-realizable-upper-bound}.

\begin{algorithm}[!ht]
\SetKwInOut{Input}{Input}
\SetKwInOut{Output}{Output}
\SetNoFillComment
\Input{Manipulation graph $G(\cX,\cE)$, hypothesis class $\mathcal{H}$}
\For{$t=1,2,\cdots$}{
    \tcc{learner commits to a classifier $h_t$ that is constructed as follows:}
    \For{$v\in \cX$}{
        \eIf{$|h\in \mathcal{H}:h(v)=+1|\geq |\mathcal{H}|/(\Delta+2)$}{
            $h_t(v)\leftarrow+1$\;
        }
        {
            $h_t(v)\leftarrow-1$\;
        }
    }
    Observe the manipulated example $v_t$ and predict according to $h_t(v_t)$\;
    \tcc{If there was a mistake:}
    \If{$h_t(v_t)\neq y_t$}{
        \eIf{$y_t=-1$}{
            $\mathcal{H}\leftarrow \mathcal{H}\setminus \{h\in \mathcal{H}:h(v_t)=+1\}$\tcp*{Remove experts that predict $v_t$ as positive.}           
        }
        {
            $\mathcal{H}\leftarrow \mathcal{H}\setminus \{h\in \mathcal{H}: \forall x\in N[v_t], h(x)=-1\}$\tcp*{Remove experts that predict $N[v_t]$ as all-negative.}
        }
    }
}
\caption{Biased majority-vote algorithm.}
\label[algo]{alg:halving}
\end{algorithm}

\begin{theorem}
\label{thm:baseline-realizable-upper-bound}
If there exists at least one perfect expert under manipulation, \Cref{alg:halving} makes at most $(\Delta+2)\ln{|\mathcal{H}|}$ mistakes.
\end{theorem}

\begin{proof}
We show whenever a mistake is made, at least $1/(\Delta+2)$ fraction of the remaining experts get excluded from $\mathcal{H}$, %
but the realizable classifier $h^\star$ is never excluded.

First, consider the case of making mistake on a true negative, i.e. $y_t=-1$. In this case, at least $|\mathcal{H}|/(\Delta+2)$ of the experts are predicting positive on $v_t$, and all of them are excluded from $\cH$. %
On the other hand, according to \Cref{prop:dominating-set}, all neighbors of $u_t$ are labeled as negative by $h^\star$. Since $v_t\in N[u_t]$, this implies that $h^\star$ must have labeled $v_t$ as negative, so $h^\star$ will not be excluded.

Next, consider the case of making a mistake on a true positive, i.e. $y_t=+1$. Since the algorithm is predicting negative on $v_t$, the agent has not moved from a different location to $v_t$ to get classified as negative. Hence, it must be the case that $v_t=u_t$. Since the agent did not move, none of the vertices in its neighborhood has been labeled positive by the algorithm, which means each of the vertices in $N[v_t]$ is labeled positive by less than $|\mathcal{H}|/(\Delta+2)$ of the experts. Since there are at most $(\Delta+1)$ vertices in $N[v_t]$, at least $|\mathcal{H}|(1-(\Delta+1)/(\Delta+2)) = |\mathcal{H}|/(\Delta+2)$ experts are predicting negative on all vertices in $N[v_t]$, all of which will be excluded. 
On the other hand, by \Cref{prop:dominating-set} again, $u_t=v_t$ is dominated by the positive region of $h^\star$, so at least one vertex in $N[u_t]$ is labeled positive by $h^\star$, which implies that $h^\star$ will not be excluded from $\cH$.

In either case, when a mistake is made, at least $1/(\Delta+2)$ fraction of the remaining experts get excluded, but the perfect expert never gets excluded. Therefore, the total number of mistakes $M=\mistake(T)$ can be bounded as follows:
\begin{align*}
&\left(1-\frac{1}{\Delta+2}\right)^M|\mathcal{H}|\ge 1
\quad\Rightarrow\quad M\leq (\Delta+2)\ln|\mathcal{H}|.
\end{align*}

\end{proof}

\paragraph{Improving the Upper Bound}
In~\Cref{sec:improving-upper-bound}, we propose a pre-processing step (\Cref{alg:improvement-halving}) that improves the mistake bound of \Cref{alg:halving} when the underlying manipulation graph is dense, i.e., the minimum degree of all the vertices is large. We achieve the following upper bound:

\begin{theorem}[Improving the number of mistakes]
\label{thm:mistake-bound-improved-halving}
\Cref{alg:improvement-halving} makes at most $\min\{n-\delta, 1+\Delta\cdot \min\{\ln|\mathcal{H}|, n-\delta-1\}\}$ mistakes, where $n=|\cX|$ and 
$\delta$ is the minimum degree of $G(\cX,\cE)$.
\end{theorem}

We leave it open to get a general instance-dependent upper bound that potentially depends on other characteristics of the manipulation graph besides the maximum/minimum degree.

\subsection{Unrealizable Case}
In the unrealizable (agnostic) case, we remove the assumption that there exists a perfect classifier under manipulation. Our goal is to design an adaptive algorithm that does not make too many mistakes compared to $\OPT$ (which is the minimum number of mistakes achieved by any classifier in $\cH$), without a priori knowledge of the value of $\OPT$ or the optimal classifier that achieves this value. 

\subsubsection{Upper Bound: Biased Weighted Majority-Vote Algorithm}
\label{sec:unrealizable}
Next, we propose an algorithm for the unrealizable (agnostic) setting. The algorithm is adapted from the \emph{weighted majority vote} algorithm, which maintains a weight for each hypothesis in $\cH$ that is initially set to be 1.
Similar to~\Cref{alg:halving}, at each round $t$, a new example arrives and gets observed as $v_t$. Let $W_+^t$ and $W_-^t$ denote the sum of weights of experts that predict $v_t$ as positive and negative respectively. Let $W_t = W_+^t+W_-^t$. If $W_+^t\geq W_t/(\Delta+2)$, the algorithm predicts %
positive, otherwise it predicts %
negative. If the algorithm makes a mistake on a true negative, then we decrease the weights of all experts that predicted $v_t$ as positive by a factor of $\gamma$. If the algorithm makes a mistake on a true positive, then we decrease the weights of all experts that labeled all the vertices in $N[v_t]$ as negative by a factor of $\gamma$. We formally present this algorithm in \Cref{alg:biased-weighted-maj-vote} and its mistake bound guarantee in \Cref{thm:biased-weighted-maj-vote-mistake-bound}.

\begin{algorithm}[!ht]
\SetKwInOut{Input}{Input}
\SetKwInOut{Output}{Output}
\SetKwInOut{Initialization}{Initialization}
\SetNoFillComment
\Input{Manipulation graph $G$, hypothesis class $\mathcal{H}$}
\Initialization{Set weights $w_0(h)\leftarrow 1$ for all classifiers $h\in \cH$. Set parameter $\gamma\gets\frac{1}{e}$.}
\For{$t=1,2,\cdots$}{
    \tcc{the learner commits to a classifier $h_t$ that is constructed as follows:}
    \For{$v\in V$}{
        Let $W_t^+(v) = \sum_{h\in\cH:h(v)=+1}w_t(h)$, $W_t^-(v) = \sum_{h\in\cH:h(v)=-1}w_t(h)$, and $W_t = W_t^+(v)+W_t^-(v) = \sum_{h\in\cH}w_t(h)$\;
        \eIf{$W_t^+(v)\geq W_t/(\Delta+2)$}{
            $h_t(v)\leftarrow +1$\;
        }
        {
            $h_t(v)\leftarrow -1$\;
        }
    }
    observe the manipulated example $v_t$ and output prediction $h_t(v_t)$\;
    \tcc{If $h_t$ makes a mistake:}
    \If{$h_t(v_t)\neq y_t$}{
        \eIf{$y_t=-1$}{
        \tcc{
            False positive mistakes: penalize the experts that label $v_t$ as positive.}
            $\mathcal{H'}\leftarrow \{h\in \cH: h(v_t)=+1\}$\;
        }        
        {
        \tcc{
                False negative mistakes: penalize the experts that label all nodes in $N[v_t]$ as negative.}
            $\mathcal{H'}\leftarrow \{h\in \cH: \forall x\in N[v_t], h(x)=-1\}$\;
        }
        if $h\in \cH'$, then $w_{t+1}(h)\leftarrow\gamma\cdot w_t(h)$;
            otherwise, $w_{t+1}(h)\gets w_t(h)$\;
    }
}
\caption{Biased weighted majority-vote algorithm.}
\label[algo]{alg:biased-weighted-maj-vote}
\end{algorithm}
\begin{theorem}
\label{thm:biased-weighted-maj-vote-mistake-bound}
\Cref{alg:biased-weighted-maj-vote} makes at most $e(\Delta+2)(\ln|\mathcal{H}|+\OPT)$ mistakes against any adversary.
\end{theorem}
We defer the proof of \Cref{thm:biased-weighted-maj-vote-mistake-bound} to \Cref{sec:proof-agnostic}. The proof follows similar high-level ideas to that of \Cref{thm:baseline-realizable-upper-bound}, but uses the total weight $W_t$ of experts, instead of the number of remaining experts, to track the progress of \Cref{alg:biased-weighted-maj-vote}. We will show that whenever a mistake is made, the total weight must decrease by a fraction of $\frac{\gamma}{\Delta+2}$.

We also consider a more general case where the input graph (denoted with $\hat{G}$) is a supergraph of the true manipulation graph $G$ that is supported on the same vertex set $\cX$. In this case, we show that by slightly modifying~\Cref{alg:biased-weighted-maj-vote} in the case of making a false negative mistake, we can obtain a similar mistake bound in terms of the maximum degree of $\hat{G}$. The modification is as follows: when making a false negative mistake on $v_t$, 
we modify~\Cref{alg:biased-weighted-maj-vote} to reduce the weights of experts that predict all-negative on an ``approximate'' neighborhood $\hat{N}[v_t]$, where $\hat{N}[v_t]$ is a subset of the neighborhood of $v_t$ under $\hat{G}$, such that $\hat{N}[v_t]$ only include vertices that are labeled as negative by $h_t$. We present the mistake bound guarantee for this algorithm in \Cref{remark:supergraph} and defer its proof to \Cref{sec:proof-agnostic}.

\begin{proposition}
\label{remark:supergraph}
When the input graph $\hat{G}$ is a supergraph of the true manipulation graph $G$ that is also supported on $\cX$, 
there exists an algorithm that makes at most $e(\Delta(\hat{G})+2)(\ln|\mathcal{H}|+\OPT)$ mistakes. Here, $\OPT$ is the minimum number of mistakes made by the optimal expert under the true manipulation graph $G$, and $\Delta(\hat{G})$ is the maximum degree of the input graph $\hat{G}$.
\end{proposition}

\subsection{Lower Bound}
\label{sec:lower-bound-deterministic}

In this section, we show lower bounds on the number of mistakes made by any \emph{deterministic} learner against an adaptive adversary in both realizable and agnostic settings. We present the lower bounds in \Cref{thm:deterministic-lower-bound}.

\begin{theorem}
\label{thm:deterministic-lower-bound}
    There exists a manipulation graph $G(\cX,\cE)$, a hypothesis class $\cH:\cX\to\cY$, and an adaptive adversary, such that any deterministic learning algorithm has to make at least $\Delta-1$ mistakes in the realizable setting and $\Delta\cdot\OPT$ mistakes in the agnostic setting, where $\OPT$ captures the minimum number of mistakes made by any classifier in the hypothesis class $\cH$. 
\end{theorem}

\begin{proof}
Here, we use the same manipulation graph $G$ and expert class $\cH$ as shown in \Cref{example:halving-fails}. The manipulation graph $G(\cX,\cE)$ is a star that includes a central node $x_0$, and $\Delta$ leaves $x_1,\cdots,x_{\Delta}$. Hypothesis set $\mathcal{H}=\{h^1,\cdots,h^{\Delta}\}$, where each $h^i\in \mathcal{H}$, assigns $+1$ to $x_i$, and $-1$ to all other nodes in $G$ (\Cref{fig:lower-bound-deterministic}). 

In the agnostic setting, we construct an adaptive adversary that always can pick a bad example $(u_t,y_t)$ on observing $h_t$, such that $h_t$ fails to classify this example correctly (i.e., $\ell(h_t,\BR_{h_t}(u_t),y_t)=1$), but this example can be successfully classified by all but one expert. The detailed construction is as follows:

\begin{enumerate}
    \item If $h_t(x_0)=+1$, then the adversary picks $(u_t = x_j,y_t=-1)$ for an arbitrary $j\in[\Delta]$. Since $u_t$ can move to $x_0$ and get classified as positive by $h_t$, we have $\ell(h_t,\BR_{h_t}(u_t),y_t)=1$. 
    On the other hand, all experts except for $h^j$ classify this example correctly.
    \item If $h_t(x)=-1$ for all nodes $x\in\cX$, then the adversary picks $(u_t=x_0,y_t=+1)$. In this case,
    $\ell(h_t,\BR_{h_t}(u_t),y_t)=1$. However, $\forall h^i\in \mathcal{H}$, we have $\BR_{h^i}(u_t)=x_0$, so this example can receive a positive classification and therefore
    $\ell(h^i,\BR_{h^i}(u_t),y_t)=0$.
    
    \item If $h_t(x_0)=-1$ and there exists $j\in[\Delta]$ such that $h_t(x_j)=+1$, then the adversary picks $(u_t=x_j,y_t=-1)$. In this case, $h_t$ will classify this example as a false positive. On the other hand, all experts except for $h^j$ will correctly classify it as negative.
\end{enumerate}
Following the above construction, the learner is forced to make a mistake at all rounds; however, in each round, at most one of the experts makes a mistake, implying that sum of the number of mistakes made by all experts is at most $T$. Since the number of experts is $\Delta$, by the pigeon-hole principle there exists an expert that makes at most $T/\Delta$ mistakes. Therefore, $\OPT\le T/\Delta$, implying a mistake lower bound of $\Delta\cdot\OPT$.

In the realizable setting, we use the same construction but only focus on the first $\Delta-1$ time steps, such that the learner is forced to make $\Delta-1$ mistakes, but there exists at least one expert that has not made a mistake so far, suppose $h^i$ is one of such experts. After the first $\Delta-1$ steps, the adversary keeps showing the same agent $(x_i,+1)$ to the learner, such that expert $h^i$ is still realizable.
\end{proof}
We remark that \Cref{thm:deterministic-lower-bound} implies no deterministic algorithm is able to make $o(\Delta)$ mistakes in the realizable setting and $o(\Delta)$ multiplicative regret in the agnostic setting. 
{Moreover, the construction shows that any deterministic algorithm is forced to err at every round in the worst-case agnostic setting, resulting in an $\Omega(T)$ regret as long as $\Delta\ge2$.}

\section{Fractional Classifiers}

\label{sec:fractional-model}
\subsection{Model}

In this section, we consider the randomized model where the learner uses a deterministic algorithm to output a probability distribution over classifiers at each round. After the learner commits to a distribution, an agent $(u_t,y_t)$ (which is chosen by an adversary) best responds to this distribution by selecting $v_t$ that maximizes the expected utility. In particular, let $P_{h_t}(v)\in[0,1]$ denote the induced probability of $h_t$ classifying node $v$ as positive, then the agent's best response function can be written as:
\begin{align}
    v_t\in \BR_{h_t}(u_t)\triangleq\arg\max_{v\in \cX} \Big[P_{h_t}(v)-\Cost(u_t,v)\Big].\label{eq:fractional-agent-util}
\end{align}
As a result of manipulation, the observable feature $v_t\in \BR_{h_t}(u_t)$ is revealed to the learner, and the learner suffers an expected loss of 
\begin{align}
    \E\left[\ell(h_t,v_t,y_t)\right]=\Pr\left[{y_t\neq h_t(v_t)}\right]=\begin{cases}
        P_{h_t}(v_t),&\text{if }{y_t=-1};\\
        1-P_{h_t}(v_t),&\text{if }{y_t=+1}.
    \end{cases}
    \label{eq:fractional-learner-util}
\end{align}
From \Cref{eq:fractional-agent-util,eq:fractional-learner-util}, we can see that the set of induced probabilities $P_{h_t}(u)$ for each $u\in\cX$ serves as a sufficient statistics for both the learner and the agent.
Therefore, instead of committing to a distribution and having the agents calculate the set of induced probabilities, the learner can directly commit to a \emph{fractional classifier} $h_t$ 
that explicitly specifies the probabilities $P_{h_t}(u)\in[0,1]$ for each $u\in\cX$. Then, after the agent best responds to these fractions and reaches $v_t$, random label $h_t(v_t)$ is realized according to the proposed probability $P_{h_t}(v_t)$.

We remark that deterministic classifiers are special cases of the fractional classifiers where $P_h(u)\in\{0,1\}$. Since the experts in $\cH$ are all deterministic, the benchmark $\OPT$, which is the minimum number of mistakes achieved by the best expert in hindsight, is still a deterministic value.

In this setting, we consider two cost functions: the \emph{weighted-graph} cost function, where the manipulation cost {from $u$ to $v$} is defined as the total weight on the shortest path {from $u$ to $v$}; and the free-edges model, where the first hop is free and the second hop costs infinity.
Recall that agents break ties by preferring features with higher expected values, so the agents in the \emph{free-edges} cost model will move to a neighbor $v_t\in N[u_t]$ with the highest probability of getting classified as positive.

    In the \Cref{sec:fractional-lower-bound}, we show that the type of randomness is of limited help because they can only reduce the $\Delta$-multiplicative regret by constants. This is evidenced by our lower bounds in \Cref{thm:fractional-onehop-lower-bound,thm:fractional-multi-hops-lower-bound}, which states that any algorithm using this type of randomness needs to suffer $\frac{\Delta}{2}$-multiplicative regret in the free-edges model and $\frac{{\Delta}}{4}$-multiplicative regret in the weighted-graph model.
    We also complement this result by providing nearly-matching upper bounds in \Cref{sec:fractional-upper-bound}.

\subsection{Lower Bound}
\label{sec:fractional-lower-bound}
\begin{theorem}
\label{thm:fractional-onehop-lower-bound}
In the model of ``free edges'' cost functions, for any sequence of fractional classifiers chosen by a deterministic algorithm, there exists an adaptive adversary such that the learner must make at least $\frac{\Delta}{2}\cdot \OPT$ mistakes in expectation.
\end{theorem}

\begin{proof}
Consider a manipulation graph $G(\cX,\cE)$ that is a star with a central node $x_0$, and $\Delta$ leaves $x_1,\cdots,x_{\Delta}$. Hypothesis set $\cH=\{h^1,\cdots,h^{\Delta}\}$, where each $h^i\in \mathcal{H}$ assigns positive to $x_i$ and negative to all other nodes in $G$, as shown in \Cref{fig:lower-bound-deterministic}.
We construct an adversary that picks $(u_t,y_t)$ upon receiving the fractional classifier $h_t$ at each round, such that $h_t$ makes a mistake with probability at least 0.5 whereas all but one expert predicts correctly. Our detailed construction is as follows:
\begin{enumerate}
 \item If $P_{h_t}(x_0)\geq 0.5$, then the adversary picks $(u_t = x_j,y_t=-1)$ for an arbitrary $j\in[\Delta]$. Since $x_0\in N[u_t]$ and $v_t$ is the node in $N[u_t]$ that achieves the largest success probability, we have $\E[\ell(h_t,\BR_{h_t}(u_t),y_t)]=P_{h_t}(v_t)\ge P_{h_t}(x_0)\ge0.5$. On the other hand, only $h^j\in \cH$ makes a mistake on $(x_j,-1)$, and all other experts classify it correctly.
 
 \item If $P_{h_t}(x_0)<0.5$ but there exists $j\in[\Delta]$ such that $P_{h_t}(x_j)\geq 0.5$, then the adversary picks $(u_t=x_j,y_t=-1)$. Since the closed neighborhood of $x_j$ only contains $\{x_j,x_0\}$, we have $v_t=u_t$ and $\E[\ell(h_t,v_t,y_t)]\geq 0.5$. In addition, all experts but $h^j$ classify this example correctly.
 
 \item If neither of the above two conditions holds, i.e.,  $P_{h_t}(v)<0.5$ for all nodes $v\in \cX$, then the adversary picks $(u_t=x_0,y_t=+1)$. In this case, no matter how the agent chooses $v_t$, the probability $P_{h_t}(v_t)$ cannot exceed 0.5. As a result,
 the learner suffers an expected loss of $\E[\ell(h_t,\BR_{h_t}(u_t),y_t)]=1-P_{h_t}(v_t)\ge 0.5$. On the other hand, all experts classify this example correctly because $x_0$ can move to the corresponding leaf node and get classified as positive.
\end{enumerate}
As a result, the learner has an expected loss of at least $0.5$ on each round, which implies
$$\E[\mistake(T)]=\E\left[\sum_{t=1}^T \ell(h_t,\BR_{h_t}(u_t),y_t)\right]\geq T/2.$$ However, in each round, at most one expert makes a mistake. Following the same arguments as \Cref{thm:deterministic-lower-bound}, we conclude that $\OPT\le \frac{T}{\Delta}$. Putting all together, the expected number of mistakes made by the learner is at least $$\E[\mistake(T)]\ge\frac{T}{2}=\frac{\Delta}{2}\cdot\frac{T}{\Delta}\ge \frac{\Delta}{2}\cdot \OPT.$$
\end{proof}

\begin{theorem}
    In weighted graphs, for any sequence of fractional classifiers chosen by a deterministic algorithm, there exists an adaptive adversary such that the learner must make at least $\frac{{\Delta}}{4}\cdot \OPT$ mistakes in expectation.
    \label{thm:fractional-multi-hops-lower-bound}
\end{theorem}
\begin{proof}
    Again, we consider the same graph structure as in \Cref{thm:fractional-onehop-lower-bound}, where $G(\cX,\cE)$ is a star with central node $x_0$ and leaf nodes $x_1,\cdots,x_{\Delta}$. Assume each edge $e\in\cE$ has the same weight $w(e)=w$, where $w\triangleq 0.5+\epsilon$ for an infinitesimal constant $\epsilon$. 
    Note that in this graph, no agent has the incentive to travel more than one edge, because it would cost them more than $1$.
    
    We work with the hypothesis set $\mathcal{H}=\{h^1,\cdots,h^{\Delta}\}$, assuming each $h^i\in \mathcal{H}$ assigns positive to $x_i$ and negative to all other nodes in $\cX$.
    We construct an adversary that picks $(u_t,y_t)$ upon receiving the fractional classifier $h_t$ as follows, such that $h_t$ makes a mistake with probability at least $\frac{1}{4}$, whereas all but one expert predicts correctly. Our detailed construction is as follows:

    Let $p=\max_{x\in\cX}P_{h_t}(x)$ denote the maximum fraction on any node. If $p<w$, then the adversary can simply pick $(u_t=x_0,y_t=+1)$, such that the learner suffers from an expected loss of 
    $$\E[\ell(h_t,\BR_{h_t}(u_t),y_t)]=1-P_{h_t}(v_t)\ge 1-p>1-w>\frac{1}{4}.$$
    As for the experts, all of them classify this agent correctly. Therefore, it suffices to consider the case of $p\ge w$ for the rest of the proof. We consider two cases depending on where $p$ is achieved:

    \begin{enumerate}
        \item If $p$ is achieved at a leaf node $x_i$ (i.e., $P_{h_t}(x_i)=p$) for some $j\in[\Delta]$, then
         the adversary chooses $(u_t=x_i,y_t=-1)$. 
         We claim that $u_t=v_t=x_i$, since the agent is already placed at the node with the highest fraction, so they do not need to pay a nonnegative cost to reach a node with an even smaller fraction. As a result, we have $\E[\ell(h_t,\BR_{h_t}(u_t),y_t)]=P_{h_t}(x_i)=p\ge w>\frac{1}{4}$.
         On the other hand, all but expert $h^i$ classify this agent correctly.
         \item If $p$ is achieved at the central node $x_0$, i.e., $P_{h_t}(x_0)=p$, then every leaf node have fractions no more than $p$. We first assume at least one leaf node $x_i$ ($i\in[\Delta]$) satisfies $P_{h_t}(x_i)< p-w$. In this case, the adversary chooses $(u_t=x_i,y_t=-1)$. 
         Since $P_{h_t}(x_0)-\Cost(u_t,x_0)=p-w > P_{h_t}(u_t)-\Cost(u_t,u_t)$, the agent will select $v_t=x_0$ as the best response, and achieve a success probability of $P_{h_t}(x_0)=p$. Therefore, the learner has expected loss $\E[\ell(h_t,\BR_{h_t}(u_t),y_t)]=p\ge w>\frac{1}{4}$. On the other hand, all but expert $x^i$ labels this agent correctly.
         \item 
         Now we consider the last case where $p$ is achieved at the central node $x_0$ and all the leaf nodes have fractions at least $p-w$. In this case, no agent has the incentive to move regardless of their initial positions. The adversary can select the next agent as follows: if $1-p\ge p-w$, then choose $(u_t=x_0,y_t=+1)$ and make the learner err with probability $1-P_{h_t}(x_0)=1-p$; otherwise, choose $(u_t=x_i,y_t=-1)$ for an arbitrary $i\in[\Delta]$ and make the learner err with probability $P_{h_t}(x_i)\ge p-w$. In either case, the learner has to suffer from an expected loss of $\E[\ell(h_t,\BR_{h_t}(u_t),y_t)]\ge\max\{1-p,p-w\}\ge\frac{1-w}{2}=\frac{1}{4}-\frac{\epsilon}{2}$.
        As for the experts, at most one of them is making a mistake.
    \end{enumerate}

Putting together all the possible cases and let $\epsilon\to0$, the learner is forced to make mistakes with probability at least $\frac{1}{4}$ on each round, i.e., $\sum_{t=1}^T \E[\ell_t(h_t,\BR_{h_t}(u_t),y_t)]\geq T/4$. However, in each round, at most one of the experts makes a mistake, implying that $\OPT\le\frac{T}{\Delta}$ as proved in \Cref{thm:deterministic-lower-bound}. As a result, the total loss made by the learner is bounded as
$$\E[\mistake(T)]=\E\left[\sum_{t=1}^T \ell(h_t,\BR_{h_t}(u_t),y_t)\right]\geq \frac{T}{4}
\ge\frac{\Delta}{4}\cdot\OPT.$$ 
The proof is thus complete.
\end{proof}

\begin{remark}
    In a related work,
    \citet{Braverman2020TheRO} %
    showed that introducing randomization in their classification rule can increase the learner's classification accuracy, and the optimal randomized classifier has the structure that agents are better off not manipulating. 
    In case (3) of the proof of \Cref{thm:fractional-multi-hops-lower-bound}, we show that even when learners choose 
    such an ``optimal'' classifier under which agents have no incentive to manipulate,
    the adversary is still able to impose a high misclassification error. This example shows the limitations of fractional classifiers.
\end{remark}

\subsection{Upper Bound}
\label{sec:fractional-upper-bound}
In this section, we show how to use the idea of \Cref{alg:biased-weighted-maj-vote} to obtain upper bounds in the randomized classifiers model.
\begin{proposition}
    In the free-edges model, \Cref{alg:biased-weighted-maj-vote} achieves a mistake  bound of 
    \[\mistake(T)\le e(\Delta+2)(\ln|\cH|+\OPT).\]
\end{proposition}
\begin{proof}
    To prove this proposition, it suffices to show that if the learner uses deterministic classifiers as a special case of fractional classifiers, then the \emph{free-edges} cost model and \emph{unweighted graph} cost model result in the same best response functions. In fact, in both models, agents manipulate their features if and only if their original nodes are labeled as negative and there exists a neighbor that is labeled as positive. Therefore, the two cost models yield the same best response behaviors to deterministic classifiers. As a result, \Cref{alg:biased-weighted-maj-vote} suffers from the mistake bound of $e(\Delta+2)(\ln|\cH|+\OPT)$.
\end{proof}

Now we consider weighted manipulation graphs.
In this setting, we can run \Cref{alg:biased-weighted-maj-vote} 
on the expanded manipulation graph $\Tilde{G}$ that is an unweighted graph constructed from $G$ by connecting all pairs $u,v$ of vertices in $G$ such that $\Cost(u,v)\leq 1$.
As a result, we obtain mistake bound in terms of $\Tilde{\Delta}$ instead of $\Delta$, where $\Tilde{\Delta}$ is the maximum degree of $\Tilde{G}$.

\begin{proposition}
\label{prop:fractional-multi-hop-upper-bound}
    Given a weighted manipulation graph $G$, running \Cref{alg:biased-weighted-maj-vote} on the expanded graph $\tilde{G}$ achieves a mistake bound of $\mistake(T)\le e(\Tilde{\Delta}+2)(\ln|\cH|+\OPT)$, where $\Tilde{\Delta}$ is the maximum degree of $\Tilde{G}$.
\end{proposition}

\begin{proof} 
After constructing $\Tilde{G}$, we can see that under any deterministic classifier, a manipulation from $u$ to $v$ happens in the weighted graph $G$ if and only if the same manipulation happens in the unweighted graph $\Tilde{G}$. Therefore, by running \Cref{alg:biased-weighted-maj-vote} on $\Tilde{G}$, we obtain a mistake bound in the original manipulation graph $G$ of $e(\Tilde{\Delta}+2)(\ln|\cH|+\OPT)$, where $\Tilde{\Delta}$ is the maximum degree of $\Tilde{G}$. 
\end{proof}

\section{Randomized Algorithms}
In this section, we propose another model of randomization. 
    Unlike the fractional classifiers model discussed in \Cref{sec:fractional-model}, we show that this randomized model induces a different type of manipulation behavior, for which success probabilities (fractions) are no longer sufficient to characterize.
    In this model, the interaction between the classifier, adversary, and agents proceeds as follows:
At each round $t$, the learner commits to a probability distribution $\cD_t$ over a set of deterministic classifiers $\{h:\cX\to\cY\}$; and promises to use $h_t\sim\cD_t$. 
Based on this mixed strategy $\cD_t$ (and before the random classifier $h_t$ gets realized), the adversary specifies the next agent to be $(u_t,y_t)$.
Then comes the most important step that differentiates this model from the {fractional classifiers} setting: the learner samples $h_t\sim \cD_t$ and \emph{releases it to the agent}, who then best responds to the true $h_t$ by modifying its features from $u_t$ to $v_t$.
The learner aims to minimize the (pseudo) regret with respect to class $\cH$:
\begin{align}
    \E\left[\regret(T)\right]\triangleq \E\left[\sum_{t=1}^T\ell(h_t,\BR_{h_t}(u_t),y_t)\right]-\min_{h^\star\in\cH}\left[\sum_{t=1}^T \ell(h^\star,\BR_{h^\star}(u_t),y_t)\right].
\end{align}

    We show that, surprisingly, releasing the random choices to the agents can help the learner to surpass the
    {$\Omega(\Delta\cdot\OPT)$} lower bound.
    In this model, we propose three algorithms that achieve $o(T)$ regret, which does not depend on $\OPT$ or $\Delta$.

\subsection{Between bandit and full-information feedback}
Before presenting our algorithms, we first investigate the feedback information available to the learner at the end of each round. 
After the agents respond, the learner observes not only the loss of the realized expert ($\ell(h_t,\BR_{h_t}(u_t),y_t)$), but also the best response state $v_t=\BR_{h_t}(u_t)$ and the true label $y_t$.
However, because the original state $u_t$ is hidden, the losses of other experts $\ell(h',\BR_{h'}(u_t),y_t)$ for $h'\neq h_t$ are not fully observable.
For this reason, the feedback structure is potentially richer than bandit feedback, which only contains  $\ell(h_t,\BR_{h_t}(u_t),y_t)$ for the realized expert $h_t$; but sparser than full-information feedback, which contains $\ell(h',\BR_{h'}(u_t),y_t)$ for all $h'\in\cH$. 

Nevertheless, we remark that the learner 
is capable of going beyond the bandit feedback using the additional information $(v_t,y_t)$. For instance, if $h'$ fully agrees with the realized $h_t$ on the entire 2-hop neighborhood of $v_t$, then $\ell(h',\BR_{h'}(u_t),y_t)=\ell(h_t,\BR_{h_t}(u_t),y_t)$. Another scenario is when the agent ends up reporting truthfully ($u_t=v_t$), so the learner can explicitly calculate the best response $\BR_{h'}(u_t)$ and the loss $\ell(h',\BR_{h'}(u_t),y_t)$ for all $h'\in\cH$.

In \Cref{sec:mixed-strategy-adaptive}, we consider a learning algorithm that discards additional information and only uses bandit feedback, which achieves $\mathcal{O}\left(\sqrt{T |\cH|\ln|\mathcal{H}|}\right)$ regret.
To remove the polynomial dependency on $|\cH|$, we propose a generic algorithmic idea that uses an all-positive classifier at random time steps to encourage the truthful reporting of agents. In this way, the learner can obtain full-information feedback on these time steps, which accelerates the learning process. 
In \Cref{sec:mixed-strategy-oblivious}, we use this idea to achieve $\mathcal{O}\left(T^{\frac{2}{3}}\ln^{\frac{1}{3}} |\mathcal{H}|\right)$ regret against any oblivious adversary. In \Cref{sec:adaptive-adversaries}, we extend this idea to the general case of adaptive adversaries and obtain a bound of $\widetilde{\mathcal{O}}\left(T^{\frac{3}{4}}\ln^{\frac{1}{4}} |\mathcal{H}|\right)$.
We also show that this framework could be useful in other strategic settings as well. For example, in \Cref{sec:strategic-perceptron}, we apply it to the setting of strategic online linear classification and obtain a mistake bound in terms of the hinge loss of the original examples when the original data points are not linearly separable.

\subsection{Algorithm based on bandit feedback}
\label{sec:mixed-strategy-adaptive}

As a warmup, we show the learner can use the vanilla EXP3 algorithm~\citep{auer2002nonstochastic}, which is a standard multi-armed bandit algorithm, to obtain sublinear regret. This algorithm works by maintaining a distribution over $\cH$ from which classifier $h_t$ is sampled, where the weights of each expert are updated according to $$p_{t+1}(h)\propto p_t(h)\cdot \exp\left(-\eta\cdot\frac{\ell(h_t,\BR_{h_t}(u_t),y_t)\cdot\indicator{h=h_t}}{p_t(h)}\right),\ \forall h\in\cH.$$
It is known that running EXP3 with learning rate $\eta=\sqrt{\frac{2\ln|\cH|}{|\cH|T}}$ will achieve a regret bound of $O(\sqrt{T|\cH|\ln|\cH|})$, see \citet{auer2002nonstochastic} for a proof.

\subsection{Algorithm based on full-information acceleration}
\label{sec:mixed-strategy-oblivious}
In this section, we provide an algorithm with $\mathcal{O}\left(T^{\frac{2}{3}}\log^{\frac{1}{3}} n\right)$ regret against \emph{oblivious adversaries}. 
An oblivious adversary is one who chooses the sequence of agents $\{(u_t,y_t)\}_{t=1}^T$ before the interaction starts, irrespective of the learner's decisions during the game.
Our algorithm (\Cref{alg:reduction-MAB-FIB}) uses a
reduction from the partial-information model to the full-information model, which is similar in spirit to \cite{awerbuch2004adaptive} and \citet[Chapter 4.6]{blum_mansour_2007}. The main idea is to divide the timeline $1,\cdots, T$ into $K$ consecutive blocks $B_1,\cdots,B_K$, where $B_j=\{(j-1)(T/K)+1,\cdots,j(T/K)\}$, and simulate a full-information online learning algorithm (Hedge) with each block representing a single step. 
Within each block $B_j$, our algorithm uses the same distribution over the experts, except that it will also pick one time-step $\tau_j\sim B_j$ uniformly at random, and assigns an all-positive classifier to $\tau_j$. 
The intention for this time step $\tau_j$ is to prevent the agent from manipulations and simultaneously obtain the loss of every expert. This observed loss then serves as an unbiased loss estimate for the average loss over the same block.
In the remainder of this section, we formally present this algorithm in \Cref{alg:reduction-MAB-FIB} and provide its regret guarantee in \Cref{thm:regret-alg-oblivious}.

\begin{algorithm}[!ht]
\SetKwInOut{Input}{Input}
\SetKwInOut{Output}{Output}
\SetNoFillComment
$K\gets T^{\frac{2}{3}}\ln^{\frac{1}{3}} |\mathcal{H}|$\;
Partition the timeline $\{1,\cdots, T\}$ into $K$ consecutive blocks $B_1,\cdots,B_K$ where $B_j=\left\{(j-1)\cdot\frac{T}{K}+1,\cdots, j\cdot\frac{T}{K}\right\}$\;
Initialize $w_1(h)\gets0,\ \forall h\in\cH$\;
\For{$1\leq j\leq K$}{
    Sample $\tau_j \in B_j$ uniformly at random\;
\For{$t\in B_j$}{
    \tcc{Commit to drawing classifier $h_t\sim\cD_t$, with $\cD_t$ defined as follows:}
    \eIf{$t=\tau_j$}{
        $\cD_t$ puts all weight on a classifier that labels every node as positive\;
    }
    {
        $\cD_t$ is a distribution over $\cH$, where $p_j(\cdot)=\frac{w_j(\cdot)}{W_j}$, $W_j=\sum_{h\in\cH} w_j(h)$\;
    }
    \tcc{Observe agent $(v_t,y_t)$.}
}
\tcc{Update the distribution at the end of $B_j$}
\For{$h\in \mathcal{H}$}{
$w_{j+1}(h)\leftarrow w_j(h) e^{-\eta \cdot\hat{\ell}_j(h)}$, where $\hat{\ell}_j(h)=\ell(h,\BR_h(v_{\tau_j}),y_{\tau_j})$\;
}

}
\caption{Randomized algorithm against oblivious adversaries 
}
\label[algo]{alg:reduction-MAB-FIB}
\end{algorithm}

\begin{theorem}
\Cref{alg:reduction-MAB-FIB}  with parameter $K=T^{\frac{2}{3}}\ln^{\frac{1}{3}} |\mathcal{H}|$ achieves a regret bound of  $\mathcal{O}\left(T^{\frac{2}{3}}\log^{\frac{1}{3}} |\mathcal{H}|\right)$ against any oblivious adversary.
\label{thm:regret-alg-oblivious}
\end{theorem}

\begin{proof}
For notational convenience, we denote the average loss over block $B_j$ as $\bar{\ell}_j(h)=\frac{\sum_{t\in B_j}\ell_t(h)}{|B_j|}$, where for each expert $h\in\cH$, $\ell_t(h)=\ell(h,\BR_h(u_t),y_t)$.
We first claim that for all $h$, $\hat{\ell}_j(h)=\ell(h,\BR_h(v_{\tau_j}),y_{\tau_j})$\footnote{Note that here $v_{\tau_j}$ is the best-response to $h_{\tau_j}$ and $\BR_h(v_{\tau_j})$ is the best-response to $h$ when the agent is at location $v_{\tau_j}$ (which we will show is the same as $u_{\tau_j}$).} is an unbiased estimator of the average loss $\bar{\ell}_j(h)$, i.e., $\E_{\tau_j\sim B_j}[\hat{\ell}_j(h)]=\bar{\ell}_j(h)$. 
This is because the algorithm predicts positive on every state at time $\tau_j$, so the agent reports truthfully ($u_{\tau_j}=v_{\tau_j}$), thus %
$\hat{\ell}_j(h)=\ell(h,\BR_h(u_{\tau_j}),y_{\tau_j})=\ell_{\tau_j}(h)$ can be observed for any expert $h$. 
Since $\tau_j$ is sampled from $B_j$ uniformly at random, we have
$$\E_{\tau_j\sim B_j}[\hat{\ell}_j(h)]=\E_{\tau_j\sim B_j}[\ell_{\tau_j}(h)]=\bar{\ell_j}(h).$$

Since the choice of $\tau_j$ is sampled independently after the distribution $p_j$ is chosen, the above claim implies that for any block $B_j$ and any $p_j$:
\begin{align}
    \E_{h\sim p_j}\left[\bar{\ell}_j(h)\right] = 
    \E_{\tau_j}\E_{h\sim p_j}\left[\hat{\ell}_j(h)\right].
    \label{eq:claim-unbiased}
\end{align}
Therefore, inside each block $B_j$ and conditioning on the history before block $B_j$, the total loss of \Cref{alg:reduction-MAB-FIB} can be bounded as follows:
\begin{align}
    \E\left[\sum_{t\in B_j}\ell_t(h_t)\right]=&\indicator{\Tilde{y}_{\tau_j}\neq y_{\tau_j}}+\sum_{t\in B_j,\ t\neq\tau_j}\E_{h\sim p_j}\left[{\ell}_t(h)\right]\nonumber\\
    \le&1+\sum_{t\in B_j}\E_{h\sim p_j}\left[{\ell}_t(h)\right]=1+|B_j|\cdot\E_{h\sim p_j}\left[\bar{\ell}_j(h)\right],\label{eq:reduction-tmp1}\\
    =&1+\frac{T}{K}\E_{\tau_j}\E_{h\sim p_j}\left[\hat{\ell}_j(h)\right].\label{eq:reduction-tmp2}
\end{align}
where the inequality \eqref{eq:reduction-tmp1} is because $\indicator{\Tilde{y}_{\tau_j}\neq y_{\tau_j}}\le1$ and the loss $\ell_t$ is always nonnegative, and \eqref{eq:reduction-tmp2} is because of the claim in \eqref{eq:claim-unbiased}. Summing over $K$ blocks and taking the expectation over $\tau_1,\cdots,\tau_K$, we obtain an upper bound of the expected total loss of \Cref{alg:reduction-MAB-FIB}:
\begin{align}
    \E\left[\sum_{t=1}^T\ell_t(h_t)\right]=&\sum_{j=1}^K\E\left[\sum_{t\in B_j}\ell_t(h_t)\right]
    \le K+\frac{T}{K}\E_{\tau_1,\cdots,\tau_K}\left[\sum_{j=1}^K\E_{h\sim p_j}\left[\hat{\ell}_j(h)\right]\right].\label{eq:reduction-tmp6}
\end{align}
From the regret guarantee of Hedge, we have that over the loss sequence $\hat{\ell}_1,\cdots,\hat{\ell}_K$, there is
    \begin{align}
        {\sum_{j=1}^K 
        \E_{h\sim p_j}\left[\hat{\ell}_j(h)\right]}
        -{\min_{h^\star \in \mathcal{H}}\sum_{j=1}^K \hat{\ell}_j(h^\star)}
        \le \mathcal{O}\Big(\sqrt{K\ln|\mathcal{H}|}\Big).
    \end{align}
Therefore, taking the expectation over $\tau_1,\cdots,\tau_K$, we obtain
\begin{align}
        \E_{\tau_1,\cdots,\tau_K}\left[{\sum_{j=1}^K 
        \E_{h\sim p_j}\left[\hat{\ell}_j(h)\right]}\right]
        \le& \E_{\tau_1,\cdots,\tau_K}\left[{\min_{h^\star \in \mathcal{H}}\sum_{j=1}^K \hat{\ell}_j(h^\star)}\right]
        + \mathcal{O}\Big(\sqrt{K\ln|\mathcal{H}|}\Big)\nonumber\\
        \le& \min_{h^\star \in \mathcal{H}}\E_{\tau_1,\cdots,\tau_K}\left[{\sum_{j=1}^K \hat{\ell}_j(h^\star)}\right]
        + \mathcal{O}\Big(\sqrt{K\ln|\mathcal{H}|}\Big)\label{eq:reduction-tmp3}\\
        =&\min_{h^\star \in \mathcal{H}}\left[{\sum_{j=1}^K \bar{\ell}_j(h^\star)}\right]
        + \mathcal{O}\Big(\sqrt{K\ln|\mathcal{H}|}\Big).\label{eq:reduction-tmp4}
    \end{align}
In the above equations, \eqref{eq:reduction-tmp3} is due to Jensen's inequality %
and \eqref{eq:reduction-tmp4} is from the unbiasedness property established in \Cref{eq:claim-unbiased}.
Finally, putting \Cref{eq:reduction-tmp6,eq:reduction-tmp4} together, and using the definition of the average loss $\bar{\ell}_j$, we conclude that
\begin{align*}
    \E\left[\sum_{t=1}^T\ell_t(h_t)\right]\le& K+\frac{T}{K}\left(\min_{h^\star \in \mathcal{H}}\left[{\sum_{j=1}^K \bar{\ell}_j(h^\star)}\right]
        + \mathcal{O}\Big(\sqrt{K\ln|\mathcal{H}|}\Big)\right)\\
        =&\min_{h^\star\in \mathcal{H}}\sum_{j=1}^K\sum_{t\in B_j}{\ell}_t(h^\star)+\mathcal{O}\left(T\sqrt{\frac{\ln|\mathcal{H}|}{K}}\right)+K.
\end{align*}

    Set $K=T^{\frac{2}{3}}\ln^{\frac{1}{3}} |\mathcal{H}|$, this gives the final regret bound of $\mathcal{O}\left( T^{2/3}\ln^{1/3}|\mathcal{H}|\right)$.
\end{proof}

\subsection{Discussion on transparency}
\label{sec:discussion-transparency}
In the sections above, we have shown that making random choices \emph{fully transparent} to strategic agents can provably help the learner to achieve sublinear regret. This is in contrast to the fractional model, where we have lower bound examples showing that keeping random choices \emph{fully opaque} to the agents leads to linear regret.
The contrasting results in these two models reveal a fundamental difference between strategic and non-strategic (adversarial) settings: unlike the adversarial setting where learners benefit more from hiding the randomness, in the strategic setting, the learner benefits more from being transparent.
At a high level, this is because the relationship between the learner and strategic agents is not completely opposing: instead, the utility of the learner and agents can align to a certain degree. 

To be more specific, in our online strategic classification setting, there are three effective players in the game: the learner who selects the classification rule, the adversary who chooses the initial features of agents, and the strategic agents who best respond to the classification rule. 
From the learner's perspective, the only malicious player is the adversary, whereas the agent has a known, controllable best response rule.
In the fractional classifiers model, both the adversary and the agents face the same amount of information (which is the set of fractions). Although the opacity can prevent the adversary from selecting worst-case agents that force the learner to err with probability $1$ (as in the lower bound examples of \ref{thm:deterministic-lower-bound}), it also reduces the learner's control of the strategic behavior of agents.
As a result, the potentially rich structure of randomness collapses to the set of deterministic fractional values, resulting in the fact that the learner is still forced to make mistakes with a constant probability.

On the contrary, in the randomized algorithms model, the learner can increase her own leverage of controlling the agents' best response behavior, and simultaneously reduces the adversary's ability of using the strategic nature of agents to hide information from the learner. Both are achieved by giving the agents more information (i.e., the realized classifiers). In other words, the learner benefits from ``colluding'' with the agents and competing against the malicious adversary in unity. This idea is demonstrated by \Cref{alg:reduction-MAB-FIB,alg:reduction-adaptive}, where the learner occasionally uses an all-positive classifier to encourage the truthful reporting of agents, thus making the adversary unable to benefit from hiding the true features from the learner.
\section{Conclusion and Open Problems}
\label{sec:open-problems}
In this paper, we studied the problem of online strategic classification under manipulation graphs. We showed fundamental differences between strategic and non-strategic settings in both deterministic and randomized models.
In the deterministic model, we show that in contrast to the nonstrategic setting where $O(\ln|\cH|)$ bound is achievable by the simple $\Halving$ algorithm, in the strategic setting, mistake/regret bounds are closely characterized by the maximum degree $\Delta$ even when $|\cH|=O(\Delta)$. In the randomized model, we show that unlike the nonstrategic setting where withholding random bits can benefit the learner, in the strategic setting, hiding the random choices has to suffer $\Omega(\Delta)$-multiplicative regret, whereas revealing the random choices to the strategic agents can provably bypass this barrier. We also design generic deterministic algorithms that achieve $O(\Delta)$-multiplicative regret and randomized algorithms that achieve $o(T)$ regret against both oblivious and adaptive adversaries.

Our work suggests several open problems. The first is to design a deterministic algorithm in the realizable setting that achieves a mistake bound in terms of generic characteristics of the manipulation graph other than the maximum degree. Recall that our upper bound of $O(\Delta\ln|\cH|)$ and lower bound of %
{$\Omega(\Delta)$} are not matching, so it would be interesting to tighten either the upper or lower bound in this setting. The second open question is to incorporate the graph structure into randomized algorithms and achieve a  {$o(T)$} 
regret bound that depends on the characterizations of the graph, such as the maximum degree.

\subsection*{Acknowledgements}
This work was supported in part by the National Science Foundation under grant CCF-2212968 {and grant CCF-2145898}, by the Simons Foundation under the Simons Collaboration on the Theory of Algorithmic Fairness, by the Defense Advanced Research Projects Agency under cooperative agreement HR00112020003, {by a C3.AI Digital Transformation Institute grant, and a Berkeley AI Research (BAIR) Commons award}.
Part of this work was conducted while KY was visiting TTIC.
The views expressed in this work do not necessarily reflect the position or the policy of the Government and no official endorsement should be inferred. Approved for public release; distribution is unlimited.
\bibliographystyle{plainnat}
\bibliography{ref}

\appendix
\section{Supplementary Materials}

\subsection{Proof of~\Cref{thm:biased-weighted-maj-vote-mistake-bound}}
\label{sec:proof-agnostic}

\medskip
\noindent\textbf{\Cref{thm:biased-weighted-maj-vote-mistake-bound}} (Restated)\textbf{.}\emph{
    \Cref{alg:biased-weighted-maj-vote} makes at most $e(\Delta+2)(\ln|\mathcal{H}|+\OPT)$ mistakes against any adversary.
}

\medskip

\begin{proof}
To begin with, we show that if a mistake is made in round $t$, then the weights get updated such that $W_{t+1}\leq W_t\big(1-\gamma/(\Delta+2)\big)$. %
Moreover, the algorithm penalizes an expert only if it made a mistake. In other words, the algorithm never over-penalizes experts who do not make a mistake.

First, suppose a mistake is made on a true negative. In this case, $v_t$ is labeled as positive by $h_t$, so the total weight of experts predicting positive on $v_t$ is at least $W_t/(\Delta+2)$, and each of their weights is decreased by a factor of $\gamma$. 
As a result, we have $W_{t+1}\leq W_t\big(1-\gamma/(\Delta+2)\big)$.
Moreover, for each classifier $h$ that gets penalized, we have $h(v_t)=+1$, so $v_t$ belongs to the positive region $S_h$, which implies that the initial node $u_t\in N[v_t]$ is able to reach the positive region $S_h$. Therefore, our previous observation indicates that $u_t$ would have ended up being predicted as positive had it best responded to $h$, so $h$ had also made a mistake.

Next, consider the case of making a mistake on a true positive. Similar to the proof of \ref{thm:biased-weighted-maj-vote-mistake-bound}, we argue that the agent has not moved from a different location to $v_t$ to get classified as negative, so $v_t=u_t$. Since the agent did not move, none of the vertices in $N[v_t]$ was labeled positive by the algorithm, implying that for each $x\in N[v_t]$, weights of experts labeling x as positive is less than $W_t/(\Delta+2)$. Therefore, taking the union of all $x\in N[v_t]$, we conclude that the total weight of experts predicting negative on all $x\in N[v_t]$ is at least $W_t\Big(1-(\Delta+1)/(\Delta+2)\Big) = W_t/(\Delta+2)$. 
All these experts are making a mistake as $v_t=u_t$ cannot reach the positive region of any of these experts, so they all end up classifying agent $u_t$ as negative. As a result, the algorithm cuts their weight by a factor of $\gamma$, resulting in $W_{t+1}\leq W_t-(\gamma W_t)/(\Delta+2)$.

Let $M=\mistake(T)$ denote the number of mistakes made by the algorithm. Since the initial weights are all set to 1, we have $W_0=|\cH|$. Together with the property that $W_{t+1}\leq W_t\left(1-\frac{\gamma}{\Delta+2}\right)$ on each mistake, we have $W_T\leq |\cH|\left(1-\frac{\gamma}{\Delta+2}\right)^M$.

Now we show that $W_T\ge \gamma^{\OPT}$. We have proved that whenever the algorithm decreases the weight of an expert, they must have made a mistake. However, it can be the case that an expert makes a mistake, but the algorithm does not detect that. In other words, the algorithm may under-penalize an expert, but it would never over-penalize. 
Let $h^\star\in\cH$ denote the best expert that achieves the minimum number of mistakes $\OPT$. Suppose the algorithm detects $q$ of the rounds where $h^\star$ makes a mistake, then we have $q\leq \OPT$. Therefore, after $T$ rounds, $W_T\ge w_T(h^\star)=\gamma^{q}\geq \gamma^{\OPT}$, since $0\leq\gamma\leq 1$. Finally, we have:
\begin{align*}
&\gamma^{\OPT}\leq W_T\leq |\mathcal{H}|\left(1-\frac{\gamma}{\Delta+2}\right)^M\\
\Rightarrow\ &\OPT\cdot\ln{\gamma}\leq \ln{|\mathcal{H}|}+M\ln{\Big(1-\frac{\gamma}{\Delta+2}\Big)}\leq \ln{|\mathcal{H}|}-M\frac{\gamma}{\Delta+2}\\
\Rightarrow\ & M\leq \frac{\Delta+2}{\gamma}\ln{|\mathcal{H}|}-\frac{\ln{\gamma}(\Delta+2)}{\gamma}\OPT\\
\end{align*}
By setting $\gamma=1/e$, we bound the total number of mistakes as $M\leq e(\Delta+2)(\ln{|\mathcal{H}|}+\OPT)$.

\end{proof}

\medskip
\noindent\textbf{\Cref{remark:supergraph}} (Restated)\textbf{.}\emph{
    When the input graph $\hat{G}$ is a supergraph of the true manipulation graph $G$, 
there exists an algorithm that makes at most $e(\Delta(\hat{G})+2)(\ln|\mathcal{H}|+\OPT)$ mistakes. Here, $\OPT$ is the minimum number of mistakes made by the optimal expert under the true manipulation graph $G$, and $\Delta(\hat{G})$ is the maximum degree of the input graph $\hat{G}$.
}

\medskip

\begin{proof}[Proof of \Cref{remark:supergraph}]
To show~\Cref{remark:supergraph} holds, we follow a similar approach that we used to prove~\Cref{thm:biased-weighted-maj-vote-mistake-bound} with some modifications in the case of false negative. 
For clarity, we use $N_{G}[v_t]$ to denote the neighborhood under the true graph $G$, and use $N_{\hat{G}}[v_t]$ to denote the neighborhood under the input graph $\hat{G}$. Since $\hat{G}$ is a supergraph of $G$, we have $N_{\hat{G}}[v_t]\supseteq N_{G}[v_t]$. Moreover, we define $\hat{N}[v_t]$ to be the subset of $N_{\hat{G}}[v_t]$ that only includes vertices labeled as negative by $h_t$.

When a false negative mistake occurs, the agent has not moved from a different location to $v_t$ to get classified as negative. Therefore, we have $v_t=u_t$. Since the agent did not move, none of the vertices in the true neighborhood $N_{G}[v_t]$ was labeled as positive by the algorithm. However, there might exist some vertices in $N_{\hat{G}}[v_t]\setminus N_{G}[v_t]$ that are labeled as positive by the algorithm.
Combined with the definition that $\hat{N}[v_t]$ includes all vertices in $N_{\hat{G}}[v_t]$ that are labeled as negative, we have
\[
    N_{G}[v_t]\subseteq \hat{N}[v_t] \subseteq N_{\hat{G}}[v_t].
\] 
According to the algorithm, for each $x\in N_{G}[v_t]$, the total weight of experts predicting $x$ as positive is less than $W_t/(\Delta(\hat{G})+2)$ where $\Delta(\hat{G})$ is the maximum degree of $\hat{G}$. Therefore, taking the union over all $x\in \hat{N}[v_t]$, it implies that the total weight of experts predicting negative on all $x\in \hat{N}[v_t]$ is at least 
\[W_t\Big(1-|\hat{N}[v_t]|/(\Delta+2)\Big)\geq W_t\Big(1-(\Delta(\hat{G})+1)/(\Delta(\hat{G})+2)\Big)=W_t/(\Delta(\hat{G})+2),\]
where the inequality comes from $\hat{N}[v_t] \subseteq N_{\hat{G}}[v_t]$. Since $N_{G}[v_t]\subseteq \hat{N}[v_t]$, all these experts are also predicting all-negative on $N_{G}[v_t]$. Therefore they are all making a mistake under the true graph $G$, as $u_t=v_t$ cannot reach the positive region of any of these experts. Reducing their weights by a factor of $\gamma$ results in $W_{t+1}\leq W_t-(\gamma W_t)/(\Delta+2)$. The rest of the proof goes through similar to the proof of~\Cref{thm:biased-weighted-maj-vote-mistake-bound}.
\end{proof}

\subsection{Improving the Upper Bound}
\label{sec:improving-upper-bound}
In this section, we propose a pre-processing step to improve the mistake bound of~\Cref{alg:halving} in some cases, depending on the structure of the underlying manipulation graph. We leave it open to get a general mistake bound that depends on other characteristics of the manipulation graph besides the maximum degree. Consider the case where the manipulation graph $G(\cX, \cE)$ is a complete graph, and the hypothesis class $\mathcal{H}$ includes all possible labelings of $\cX$, i.e. $|\mathcal{H}|=2^{|\cX|}$. However,~\Cref{prop:effective-hypothesis-class-complete-graphs} shows that all the examples $(u_t,y_t)$ arriving over time get labeled the same: either all positive or all negative. Therefore, the size of the \emph{effective} hypothesis class is $2$.

\begin{proposition}
\label{prop:effective-hypothesis-class-complete-graphs}
If the manipulation graph $G(\cX,\cE)$ is a complete undirected graph, then all the examples arriving over time are labeled the same, i.e. all positive or all negative.
\end{proposition}

\begin{proof}
Consider a hypothesis $h$ that labels at least one node $v\in \cX$ as positive. Then any example $u_t$ arriving at time-step $t$ can reach $v$ and get classified as positive. Hence, $h$ classifies all the examples as positive. On the other hand, if $h$ labels all the nodes $v\in \cX$ as negative, then it would classify all the examples arriving over time as negative.
\end{proof}

\Cref{alg:halving} has a mistake bound of $\mathcal{O}(\Delta \ln|\mathcal{H}|)$ in the realizable case. However, when the manipulation graph is complete, we can get a mistake bound of $1$ as follows: initially starting with an all-positive classifier, if a mistake happens, switch to an all-negative classifier. The case of complete graphs shows that depending on the underlying manipulation graph, there can be a large gap between the upper bound given by~\Cref{alg:halving} and the best achievable bound. ~\Cref{alg:improvement-halving} is a pre-processing step to improve this gap. %

\begin{algorithm}
\SetKwInOut{Input}{Input}
\SetKwInOut{Output}{Output}
\SetNoFillComment
\Input{$G(\cX,\cE)$, hypothesis class $\mathcal{H}$}
\For{$t=1,\cdots,T$}{
Commit to $h_t$ that labels all nodes as positive\;
\tcc{Observe $(v_t,y_t)$}
\If{$y_t\neq h_t(v_t)$}{\tcc{when the first mistake happens, remove all the hypotheses that make a mistake}
$\mathcal{H}'\gets\mathcal{H}\setminus \{h:\ \exists v\in N[v_t],\ h(v)=+1\}$\;
Break\;
}
}
Run~\Cref{alg:halving} on $(G, \mathcal{H}')$\;
\caption{A pre-processing step to improve the mistake bound of~\Cref{alg:halving}}
\label[algo]{alg:improvement-halving}
\end{algorithm}

\Cref{alg:improvement-halving} initially starts with an all-positive classifier. When the first mistake happens on a node $v_t$, it means that $v_t$ and all its neighbors need to be classified as negative. Hence, we exclude all the hypothesis $h\in \mathcal{H}$ that classify any node $v\in N[v_t]$ as positive from $\mathcal{H}$. After filtering $\mathcal{H}$, we run~\Cref{alg:halving} on the new set $\mathcal{H}'$. 
We now restate the guarantee of \Cref{alg:improvement-halving} that we presented in \Cref{thm:mistake-bound-improved-halving}, and show its proof.

\medskip
\noindent\textbf{\Cref{thm:mistake-bound-improved-halving}} (Restated)\textbf{.}\emph{
    \Cref{alg:improvement-halving} makes at most $\min\{n-\delta, 1+\Delta\cdot \min\{\ln|\mathcal{H}|, n-\delta-1\}\}$ mistakes, where $n=|\cX|$ and 
    $\delta$ is the minimum degree of $G(\cX,\cE)$.
}
\medskip

\begin{proof}
After the first mistake happens on $v_t$,~\Cref{alg:improvement-halving} only keeps the hypotheses that label all the nodes in $N[v_t]$ as negative. Since $\big|N[v_t]\big|\geq \delta+1$, the number of such hypotheses is at most $2^{n-(\delta+1)}$. Therefore, the filtered-out hypothesis set $\mathcal{H'}$ satisfies $|\mathcal{H}'|\leq \min\{|\mathcal{H}|, 2^{n-\delta-1}\}$. Therefore, the number of mistakes that~\Cref{alg:halving} makes on the filtered hypothesis set is at most:
\[1+\Delta\cdot\ln(|\mathcal{H}'|)\leq1+\Delta\cdot\min\{\ln|\mathcal{H}|, \ln(2^{n-\delta-1}) = 1+\Delta\cdot\min\{\ln|\mathcal{H}|, n-\delta-1\}\]

Suppose that $n-\delta<1+\Delta\cdot\min\{\ln|\mathcal{H}|, n-\delta-1\}$. After the first mistake happens on $v_t$ and the labels of $N[v_t]$ get flipped, for the remaining graph $G\setminus N[v_t]$, the labels can get flipped one by one whenever a mistake is observed. Therefore, the total number of mistakes is at most:
\[\min\{n-\delta, 1+\Delta\cdot \min\{\ln|\mathcal{H}|, n-\delta-1\}\}\]
 
\end{proof}

\begin{remark}
When the manipulation graph is dense, the mistake bound in \Cref{thm:mistake-bound-improved-halving} can greatly outperform that given in \Cref{thm:baseline-realizable-upper-bound}. For instance, in complete graphs where both the minimum degree and the maximum degree are $n-1$, \Cref{thm:mistake-bound-improved-halving} guarantees that \Cref{alg:improvement-halving} makes at most one mistake, whereas \Cref{alg:halving} could end up making $n$ mistakes in total, one on each vertex.
\end{remark}

\subsection{Extension of the Deterministic Model to Directed Manipulation Graphs}
\label{sec:directed-graphs}

Suppose that the manipulation graph $G(\cX,\cE)$ is a directed graph. We show how to modify~\Cref{alg:halving,alg:biased-weighted-maj-vote} to work in the case of directed manipulation graphs and get a regret bound that depends on $\Delta_{\text{out}}$ instead of $\Delta$, where $\Delta_{\text{out}}$ is the maximum out-degree of all the nodes $v\in\cX$.
\begin{proposition}
In the realizable case,~\Cref{alg:halving} can be modified to make at most $(\Delta_{\text{out}}+2)\ln |\mathcal{H}|$ mistakes. 
\end{proposition}
\begin{proof}
First, we need to change the threshold of the majority vote for classifying a node as positive from $1/(\Delta+2)$ to $1/(\Delta_{\text{out}}+2)$. Now, if a mistake on a true negative happens, then $1/(\Delta_{\text{out}}+2)$ of the remaining hypotheses gets discarded, which are the set of  experts that predict the observable node as positive. On the other hand, if a mistake on a true positive happens, it means that the agent was classified as negative and did not move. Therefore, all the nodes in the reachable out-neighborhood were classified as negative by the algorithm. The number of reachable nodes from the starting node is at most $\Delta_{\text{out}}+1$, and for each of them. less than $(1/(\Delta_{\text{out}}+2))|\mathcal{H}|$ experts classified them as positive. Therefore, a total of $|\mathcal{H}|\Big(1-(\Delta_{\text{out}}+1)/(\Delta_{\text{out}}+2)\Big)=(1/(\Delta_{\text{out}}+2))|\mathcal{H}|$ remaining hypotheses are classifying the entire reachable set as negative, and they are all making a mistake. As a result, whenever a mistake happens, $1/(\Delta_{\text{out}}+2)$ fraction of the hypotheses can get discarded. This results in a mistake bound of $(\Delta+2)\ln{|\mathcal{H}|}$.
\end{proof}

Similarly, we can show that \Cref{alg:biased-weighted-maj-vote} can be modified to get a mistake bound that depends on $\Delta_{\text{out}}$ instead of $\Delta$, as shown in the following proposition. 
\begin{proposition}
In the unrealizable case,~\Cref{alg:biased-weighted-maj-vote} can be modified to make at most $e(\Delta_{\text{out}}+2)(\ln |\mathcal{H}|+\OPT)$ mistakes. 
\end{proposition}

\subsection{Regret bound of $\widetilde{\mathcal{O}}\left(T^{\frac{3}{4}}\ln^{\frac{1}{4}} |\mathcal{H}|\right)$ against an adaptive adversary}
\label{sec:adaptive-adversaries}

In this section, we present an algorithm (\Cref{alg:reduction-adaptive}) based on the idea of full-information acceleration, and prove a regret bound of $\widetilde{\mathcal{O}}\left(T^{\frac{3}{4}}\ln^{\frac{1}{4}} |\mathcal{H}|\right)$ against general adaptive adversaries in \Cref{thm:regret-alg-adaptive-reduction}. The proof of this theorem requires a more careful analysis of the difference between the estimated loss sequence and the actual loss sequence using martingale difference sequences, which borrows similar ideas from \citet{mcmahan2004online}.

\begin{algorithm}[!ht]
\SetKwInOut{Input}{Input}
\SetKwInOut{Output}{Output}
\SetNoFillComment
Initialize $w_1(h)\gets0,\ \forall h\in\cH$\;
Initialize step size $\eta\gets \sqrt{\frac{8\ln|\cH|}{T}}$, exploration coefficient $\gamma\gets T^{-\frac{1}{4}}\ln^{\frac{1}{4}}(T|\cH|)$\;
Let $h^+$ be an all-positive classifier\;
\For{$t\in [T]$}{
\tcc{Commit to a distribution $\cD_t$ defined as follows, then draw classifier $h_t\sim\cD_t$}
    Let $\cD_t$ be a distribution over $\cH\cup\{h^+\}$ specified by probabilities $p_t(h^+)=\gamma$, and $p_t(h)=(1-\gamma)\frac{w_t(h)}{W_t}$ for all $h\in\cH$, where $W_t=\sum_{h'\in\cH} w_t(h')$\;
    \tcc{Observe agent $(v_t,y_t)$.}
    \tcc{Construct an estimated loss vector and use it to update the weights:}
    \For{$h\in \mathcal{H}$}{
    $\hat{\ell_t}(h)\gets\frac{\ell(h,\BR_h(v_t),y_t)\cdot\indicator{h_t=h^+}}{\gamma}$\;
    $w_{t+1}(h)\leftarrow w_t(h) e^{-\eta \cdot\hat{\ell}_t(h)}$.
    }
}
\caption{Randomized algorithm against adaptive adversaries}
\label[algo]{alg:reduction-adaptive}
\end{algorithm}

\begin{theorem}
    \Cref{alg:reduction-adaptive} achieves a regret of $\widetilde{\mathcal{O}}\left(T^{\frac{3}{4}}\ln^{\frac{1}{4}} |\mathcal{H}|\right)$ against any adaptive adversary.
    \label{thm:regret-alg-adaptive-reduction}
\end{theorem}
\begin{proof}
    Similar to the proof of \Cref{thm:regret-alg-oblivious}, we first show that at every round $t$ and for all experts $h\in\cH$, $\hat{\ell}_t(h)$ is an unbiased estimate of $\ell_t(h)$, where $\ell_t(h)=\ell(h,\BR_{h}(u_t),y_t)$ is the true loss of $h$. 
    Since we are dealing with adaptive adversaries, we show that for every $h\in\cH$, $\left(\hat{\ell}_t(h)-\ell_t(h)\right)_{t=1}^T$ is a Martingale Difference Sequence: let $\cF_t$ denote the $\sigma$-algebra generated by the randomness up to time $t$, then
    \begin{align}
        \E\left[\left.\hat{\ell_t}(h)-\ell_t(h)\right|\cF_{t-1}\right]=&\E\left[\left.\gamma\cdot \frac{\ell(h,\BR_h(v_t),y_t)}{\gamma}
        -\ell(h,\BR_h(u_t),y_t)\ 
        \right|\cF_{t-1}\right]=0.
        \label{eq:adaptive-unbiased}
    \end{align}
    Here, the first equality is because 
    $\hat{\ell}_t\neq0$ only when $h_t=h^+$ is an all-positive classifier, which happens with probability $\gamma$.
    The second equality is because the agent would not move under $h^+$, resulting in $u_t=v_t$. Moreover, from the definition of $\hat{\ell}_t$, the term $\hat{\ell}_t(h)-\ell_t(h)$ is bounded in absolute value by $\frac{1}{\gamma}$.
    Now we calculate the expected cumulative loss of \Cref{alg:reduction-adaptive}. 
    \begin{align}
        \E\left[\sum_{t=1}^T \ell_t(h_t)\right]=&
        \E\left[\sum_{t=1}^T \E\left[\ell_t(h_t)|\cF_{t-1}\right] \right]\label{eq:tower-property}\\
        =&\E\left[\sum_{t=1}^T \E\left[\left.\gamma\cdot\indicator{y_t\neq1}+ (1-\gamma)\cdot
        \E_{h\sim \frac{w_t(\cdot)}{W_t}}\left[\ell_t(h)\right]\ \right|\ \cF_{t-1}\right] \right]\nonumber\\
        \le &\E\left[\sum_{t=1}^T \gamma+\E_{h\sim p_t'}\left[\ell_t(h)\right] \right],\qquad\qquad\qquad\small{ \text{where }p_t'(h)\triangleq \frac{w_t(h)}{W_t},\ \forall h\in\cH};
        \label{eq:red-tmp-1}\\
        =&\gamma T+ \E\left[\sum_{t=1}^T \E_{h\sim p_t'}\left[\hat{\ell}_t(h)\right] \right]+\E\left[\sum_{t=1}^T \E_{h\sim p_t'}\left[\ell_t(h)-\hat{\ell}_t(h)\right] \right].
        \label{eq:red-tmp-2}
    \end{align}
    In the above equations, \Cref{eq:tower-property} is from the tower property of conditional expectations,
    \Cref{eq:red-tmp-1} is because $\indicator{y_t\neq1}\le1$ and $1-\gamma\le1$, where we also use the tower property to remove the conditional expectations. Finally, \Cref{eq:red-tmp-2} is because we add and subtract the second term. 
    
    Now, for the third term in \eqref{eq:red-tmp-2}, note that $p_t'$ is defined on $\cF_{t-1}$, so $\left(\E_{h\sim p_t'}\left[\ell_t(h)-\hat{\ell}_t(h)\right]\right)_{t=1}^T$ is also a martingale difference sequence with respect to the filtration $\left(\cF_{t}\right)_{t=1}^T$. Again, from the tower property, this term is always zero:

    \begin{align}
        \E\left[\sum_{t=1}^T \E_{h\sim p_t'}\left[\ell_t(h)-\hat{\ell}_t(h)\right] \right]=\E\left[\sum_{t=1}^T \E\left[\E_{h\sim p_t'}\left[\left.\ell_t(h)-\hat{\ell}_t(h)\ \right|\cF_{t-1}\right]\right] \right]=0.
        \label{eq:concentration-loss}
    \end{align}
    Since $p_1',\cdots,p_T'$ are exactly the same as the strategies generated by running Hedge on the estimated loss sequence $\hat{\ell}_1,\cdots,\hat{\ell}_T$, and the magnitude of the losses are all bounded by $\frac{1}{\gamma}$, we have the following regret guarantee from~\citet{freund1997decision}:
    \begin{align}
        \E\left[\sum_{t=1}^T \E_{h\sim p_t'}\left[\hat{\ell}_t(h)\right] -\min_{h^\star\in\cH}\sum_{t=1}^T \hat{\ell}_t(h^\star)\right]\le {\mathcal{O}}\left(\frac{1}{\gamma}\sqrt{T\ln|\cH|}\right).\label{eq:hedge-guarantee}
    \end{align}
    Putting \Cref{eq:red-tmp-2,eq:concentration-loss,eq:hedge-guarantee} together gives us the bound on expected loss:
    \begin{align}
        \E\left[\sum_{t=1}^T \ell_t(h_t) \right]\le &\E\left[\min_{h^\star\in\cH}\sum_{t=1}^T \hat{\ell}_t(h^\star)\right] +{\mathcal{O}}\left(\gamma T+\frac{1}{\gamma}\sqrt{T\ln|\cH|}\right).\label{eq:tmp10}
    \end{align}
    We define
    \begin{align*}
        \widehat{\OPT}\triangleq\min_{h^\star\in\cH}\sum_{t=1}^T \hat{\ell}_t(h^\star),
    \end{align*}
    then the above inequality \eqref{eq:tmp10} implies 
    \begin{align}
        \E[\regret]=\E\left[\sum_{t=1}^T \ell_t(h_t) -\OPT\right]\le \E\left[\widehat{\OPT}-\OPT\right]+{\mathcal{O}}\left(\gamma T+\frac{1}{\gamma}\sqrt{T\ln|\cH|}\right).
        \label{eq:tmp23}
    \end{align}
    Now, the last step is to bound the expected difference between $\widehat{\OPT}$ and the true optimal $\OPT=\min_{h^\star\in\cH}\ell_t(h^\star)$. We have:
    \begin{align*}
        \E\left[\widehat{\OPT}-\OPT\right]
        =\E\left[\min_{\hat{h}}\max_{h}\sum_{t=1}^T \hat{\ell}_t(\hat{h})-\ell_t(h)\right]
        \le\E\left[\max_{h\in\cH} \sum_{t=1}^T\hat{\ell}_t(h)-\ell_t(h)\right].
    \end{align*}
    Since $\left(\hat{\ell}_t(h)-\ell_t(h)\right)_{t=1}^T$ is a martingale difference sequence for any fixed $h$, we use Azuma-Hoeffding inequality together with the union bound to obtain
    \begin{align*}
        \Pr\left[\max_{h\in\cH} \sum_{t=1}^T\hat{\ell}_t(h)-\ell_t(h)\ge\frac{1}{\gamma}\sqrt{2T\ln\left(\frac{1}{\delta}\right)}\right]\le\delta|\cH|.
    \end{align*}
    Setting $\delta=\frac{1}{T|\cH|}$ gives us
    \begin{align}
        \E\left[\widehat{\OPT}-\OPT\right]\le\E\left[\max_{h\in\cH} \sum_{t=1}^T\hat{\ell}_t(h)-\ell_t(h)\right]\le \frac{1}{\gamma}\sqrt{2T\ln\left(\frac{1}{\delta}\right)}+\delta|\cH|\cdot T\le\mathcal{O}\left(\frac{1}{\gamma}\sqrt{2T\ln\left(T|\cH|\right)}\right).
        \label{eq:opt-concentration}
    \end{align}
    Finally, by putting \Cref{eq:tmp23,eq:opt-concentration} together, and setting 
    $\gamma=T^{-\frac{1}{4}}\ln^{\frac{1}{4}}(T|\cH|)$,
    we {derive} the desired regret bound:
    \begin{align*}
        \E[\regret]=\E\left[\sum_{t=1}^T \ell_t(h_t) -\OPT\right]\le \mathcal{O}\left(T^{\frac{3}{4}}\ln^{\frac{1}{4}}(T|\cH|)\right).
    \end{align*}
    
\end{proof}

\subsection{Strategic online linear classification}
\label{sec:strategic-perceptron}

In this section, we propose an algorithm for the {problem of online} linear classification in the presence of strategic behavior. 
{In this setting,} each {original }example $\z_t$ can move for an $\ell_2$ distance of at most $\alpha$ and reach a new observable state $\x_t$; and the examples would move for a minimum distance that results in a positive classification.
{\citet{ahmadi2021strategic} propose an algorithm for the case that} 
{original} examples are linearly separable; 
{in the case of inseparable examples, they get a mistake bound in terms of the hinge loss of \emph{manipulated} examples, and }leave it as an open problem 
to obtain a mistake bound in terms of the hinge-loss of \emph{original} examples.

In this section, we propose an algorithm 
for the inseparable case {that obtains a bound in terms of the hinge-loss of \emph{original} examples. However, our mistake bound has an additional $\mathcal{O}(\sqrt{T})$ additive term compared to the bound obtained by \citet{ahmadi2021strategic} in the separable case}. The idea behind this algorithm is to use an all-positive classifier at random time steps to observe the un-manipulated examples. 
Using the un-manipulated examples, the standard Perceptron algorithm suffices to deal with inseparable data.
For simplicity, we present the algorithm for oblivious adversaries and remark that a similar bound could be obtained for the case of adaptive adversaries using similar techniques as in \Cref{sec:adaptive-adversaries}.

\begin{algorithm}[!ht]
\SetKwInOut{Input}{Input}
\SetKwInOut{Output}{Output}
\SetNoFillComment
Partition the timeline $1,\cdots, T$ into $K$ consecutive blocks $B_1,\cdots,B_K$ where $B_j=[\frac{(j-1)T}{K}+1,\frac{jT}{K}]$\;
Initialize $w_1\gets\mathbf{0}$\;
\For{$j\in[K]$}{
    Sample $\tau_j \in B_j$ uniformly at random\;
\For{$t\in B_j$}{
    \eIf{$t=\tau_j$}{
        Use classifier $h_t\gets h^+$, where $h^+(\x)=+1\ \forall \x$\;
    }
    {
        Use classifier $h_t\gets h^j$, where $h^j(\x)=\text{sgn}\left(\frac{\w_j^\transpose \x}{|\w_j|}-\alpha\right)$\;
    }
    \tcc{Observe example $(\x_t,y_t)$}
}
\If{$y_{\tau_j}\neq h_{\tau_j}(\x_{\tau_j})$}{
$\w_{j+1}\gets {\w_j}+y_{\tau_j}\x_{\tau_j}$\;
}
}
\caption{Algorithm for online linear strategic classification when original examples are inseparable}
\label[algo]{alg:reduction-MAB-FIB-hinge-loss}
\end{algorithm}

\begin{theorem}
Let $S=\{(\z_t,y_t)\}_{t=1}^T$ be the set of original data points, where  $\max_t|\z_t|\le R $. For any $\w^\star$, \Cref{alg:reduction-MAB-FIB-hinge-loss} with parameter $K=\sqrt{T} R\|\w^\star\|$ satisfies
\begin{align}
    \E\left[\mistake(T)\right]\le 2L_{\hinge}(\w^\star,S)+2\sqrt{T}R\|\w^\star\|,
\end{align}
where the hinge loss is defined as
$$L_{\hinge}(\w^\star,S)\triangleq \sum_{(z_t,y_t)\in S} \max\left\{0,1-y_t (\z_t^\transpose \w^\star)\right\}.$$
\end{theorem}
\begin{proof}
We use $\ell_t(h)=\indicator{y_t\neq h(\BR_{h}(\z_t)}$ to denote the loss of classifier $h$ had agent $(\z_t,y_t)$ best responded to $h$. 

    In each block $B_j$, we have $\ell(h_{\tau_j})=\indicator{y_{\tau_j}\neq +1}\le1$ on the all-positive step $\tau_j$. On the other steps $t\neq \tau_j$, since $h^j$ is obtained by shifting the boundary $\w_j$ by $\alpha$,
    an agent $(\x_t)$ can reach the positive region of $h^j$ if and only if its original features $(\z_t)$ have a nonnegative dot product with $\w_j$. Thus we have
    $$\ell_t(h^j)=\indicator{h^j(\x_t)\neq y_{t}}=\indicator{\text{sgn}\left(\frac{\w_j^\transpose \x_t}{|\w_j|}-\alpha\right)\neq y_{t}}
    =\indicator{\text{sgn}\left(\frac{\w_j^\transpose \z_t}{|\w_j|}\right)\neq y_t}=\indicator{\text{sgn}\left({\w_j^\transpose \z_t}\right)\neq y_t}.$$

    As a result, we can bound the number of mistakes as follows:
    \begin{align}
        \E\left[\mistake(T) \right]
        =&
        \E\left[\sum_{j=1}^K
        \sum_{t\in B_j}\ell_t(h_t)\right]
        {\le}\E\left[\sum_{j=1}^K\left(1+
        \sum_{t\in B_j} 
        \ell_t(h^j)
        \right)\right]\nonumber\\
        {=}& K+\sum_{j=1}^K \sum_{t\in B_j} \indicator{\text{sgn}\left({\w_j^\transpose \z_t}\right)\neq y_t}\nonumber\\
        {\le}& K+\frac{T}{K}\sum_{j=1}^K
        \E_{\tau_j\sim B_j}\indicator{\text{sgn}\left({\w_j^\transpose \z_{\tau_j}}\right)\neq y_{\tau_j}},\label{tmp::5}
    \end{align}
    where the last step is because $\tau_j$ is sampled uniformly at random from $B_j$.

    Note that $w_1,\cdots,w_K$ is obtained from running the standard Perceptron algorithm on examples $S_\tau\triangleq\{(\x_{\tau_1},y_{\tau_1}),\cdots, (\x_{\tau_K},y_{\tau_K})\}$.
    Since at each $\tau_j$, the learner uses an all-positive classifier to stop the agents from moving, we have $\x_{\tau_j}=\z_{\tau_j}$, and $S_\tau=\{(\z_{\tau_1},y_{\tau_1}),\cdots, (\z_{\tau_K},y_{\tau_K})\}$.
    From \citet{block1962perceptron}, we have
    $$\sum_{j=1}^K
        \indicator{\text{sgn}\left({\w_j^\transpose \z_{\tau_j}}\right)\neq y_{\tau_j}}\le R^2\|\w^\star\|^2+2L_{\hinge}(\w^\star,S_\tau).$$
    Taking the expectation over $\tau_1,\cdots,\tau_K$, we have the following mistake bound on the standard perceptron algorithm:
    
    \begin{align}
        \frac{T}{K}\sum_{j=1}^K
        \E_{\tau_j\sim B_j}\indicator{\text{sgn}\left({\w_j^\transpose \z_{\tau_j}}\right)\neq y_{\tau_j}}\le& \frac{T}{K}R^2\|\w^\star\|^2+2
        \frac{T}{K}\sum_{j=1}^K
        \E_{\tau_j\sim B_j}
        L_{\hinge}(\w^\star,(\z_{\tau_j},y_{\tau_j}))\nonumber\\
        =&\frac{T}{K}R^2\|\w^\star\|^2+2
        \frac{T}{K}\sum_{j=1}^K \frac{1}{|B_j|}\sum_{t\in B_j}L_{\hinge}(\w^\star,(\z_t,y_t))\label{tmp::8}\\
        =&\frac{T}{K}R^2\|\w^\star\|^2+2
        \sum_{j=1}^K \sum_{t\in B_j}L_{\hinge}(\w^\star,(\z_t,y_t))\label{tmp::4}\\
        =&\frac{T}{K}R^2\|\w^\star\|^2+2 L_{\hinge}(\w^\star,S).\label{tmp::3}
    \end{align}
    In the above inequalities, \Cref{tmp::8} follows from the fact that $\tau_j$ is distributed uniformly at random in block $B_j$, and \Cref{tmp::4} is because every block has size $|B_j|=\frac{T}{K}$.
    
    Now we plug \Cref{tmp::3} back into \eqref{tmp::5} and obtain
    \begin{align*}
        \E\left[\mistake(T) \right]\le K+\frac{T}{K}R^2\|\w^\star\|^2+2L_{\hinge}(\w^\star,S).
    \end{align*}
    Finally, letting $K=\sqrt{T} R\|\w^\star\|$ yields the desired bound.
\end{proof}

\subsection{Two populations}
\label{sec:two-populations}

In this section, we study extensions of the unit-edge cost function in our baseline model. We assume there are two populations with different manipulation costs: agents of group $A$ face a cost of $0.5$ on each edge, whereas agents of group $B$ face a cost of $1$. As a result, in response to deterministic classifiers, agents from group $A$ move within their two-hop distance neighborhood, whereas agents from group $B$ only move inside their one-hop distance neighborhood.

We suppose each agent has fixed probabilities of belonging to each group, regardless of the initial position and the label chosen by the adversary. In other words, at every round $t$, after the adversary picks the next agent $(u_t,y_t)$, we assume nature independently assigns this agent to group $c_t=B$ with probability $\beta$ and $c_t=A$ with probability $\alpha=1-\beta$. The agent's best response to classifier $h_t$ is a function of $u_t$ and $c_t$:

\begin{align*}
    v_t\in\BR_{h_t}(u_t,c_t)\triangleq\begin{cases}
        \arg\max_{v\in \cX} \Big[\Value(h_t(v))-\Cost_A(u_t,v)\Big], & \text{if } c_t=A\\
        \arg\max_{v\in \cX} \Big[\Value(h_t(v))-\Cost_B(u_t,v)\Big], & \text{if } c_t=B.
    \end{cases}
\end{align*}
As a result of manipulation, the learner suffers loss $\ell(h_t,v_t,y_t)=\ell(h_t,\BR_{h_t}(u_t,c_t),y_t)$ and observes $(v_t,y_t)$ together with group membership $c_t$. The learner's goal is to bound the expected number of mistakes in terms of the optimal number of mistakes in expectation, where the expectations are taken over the random group assignments and the possible randomness in the learning algorithm and the adversary's choices. 
\begin{align*}
    \E[\mistake(T)]=\E\left[\sum_{t=1}^T\ell(h_t,\BR_{h_t}(u_t,c_t),y_t)\right],\quad
    \E[\OPT]=\min_{h\in\cH}\E\left[\sum_{t=1}^T\ell(h,\BR_{h}(u_t,c_t),y_t)\right].
\end{align*}

 {We propose \Cref{alg:two-populations} that} is based on the idea of biased weighted majority vote (\Cref{alg:biased-weighted-maj-vote}), with
a \emph{group-independent} threshold for the biased majority votes, and
a \emph{group-dependent} way of penalizing experts. We state the mistake bound guarantee in \Cref{thm:two-populations}.
\begin{algorithm}[t]
    \SetKwInOut{Input}{Input}
    \SetKwInOut{Output}{Output}
    \SetNoFillComment
    \Input{Manipulation graph $G(\cX,\cE)$, hypothesis class $\mathcal{H}$}
    Set initial weights $w_1(h)\leftarrow 1$ for all experts $h\in \cH$\;
    Set discount factor $\gamma=\frac{1}{e}$, threshold $\theta=\max\left\{\frac{1}{\Delta+1+\frac{1}{\beta}},\ \frac{1}{\Delta^2+2}\right\}$\;
    \For{$t=1,2,\cdots$}{
        \tcc{The learner commits to a classifier $h_t$ that is constructed as follows:}
        \For{$v\in V$}{
            Let $W_t^+(v) = \sum_{h\in\cH:h(v)=+1}w_t(h)$, $W_t^-(v) = \sum_{h\in\cH:h(v)=-1}w_t(h)$, and $W_t= W_t^+(v)+W_t^-(v)$\;
            \eIf{$W_t^+(v)\geq \theta\cdot W_t$}{
                $h_t(v)\leftarrow +1$\;
            }
            {
                $h_t(v)\leftarrow -1$\;
            }
        }
        \tcc{Unlabeled example $v_t$ is observed.}
        output prediction $h_t(v_t)$\;
        \tcc{The true label $y_t$ and group membership $c_t$ are observed}
        \tcc{If there was a mistake:}
        \If{$h_t(v_t)\neq y_t$}{
            \eIf{$h_t(v_t)=+1$}{
               for all $h\in \mathcal{H}:h(v_t)=+1$, $w_{t+1}(h)\leftarrow \gamma\cdot w_t(h)$\tcp*{false positive} 
            }
            {
                \eIf{$c_t=A$}{$\mathcal{H'}\leftarrow \{h\in \cH: \forall x\in N^2[v_t], h(x)=-1\}$\tcp*{$N^2[\cdot]$ is the 2-hop neighborhood}}
                {$\mathcal{H'}\leftarrow \{h\in \cH: \forall x\in N[v_t], h(x)=-1\}$}
                If $h\in \cH'$, $w_{t+1}(h)\leftarrow\gamma\cdot w_t(h)$, otherwise $w_{t+1}(h)\leftarrow w_t(h)$.\tcp*{false negative}
            }
        }
    }
    \caption{Biased weighted majority-vote algorithm for two populations.}
    \label[algo]{alg:two-populations}
    \end{algorithm}

    \begin{theorem}
        \label{thm:two-populations}
        In the setting of two populations and population $B$ has probability $\beta$, \Cref{alg:two-populations} achieves an expected mistake bound of the following:
        \begin{align*}
            \E[\mistake(T)]\le e\cdot\min\left\{\Delta+1+\frac{1}{\beta},\ \Delta^2+2\right\}\left(\ln|\cH|+\E[\OPT]\right).
        \end{align*}
    \end{theorem}

\begin{remark}    
    In \Cref{thm:two-populations}, when all agents can make two hops (i.e., $\beta=0$), the mistake bound reduces to the guarantee provided \Cref{thm:biased-weighted-maj-vote-mistake-bound} with $\Delta^2$ as the maximum degree. In this case, \Cref{alg:two-populations} is equivalent with \Cref{alg:biased-weighted-maj-vote} running on the expanded neighborhood graph $\widetilde{G}$ in which every two nodes of distance at most two are connected by an edge. Here, $\Delta^2$ is an upper bound on the maximum degree of $\widetilde{G}$.In contrast, when all agents can only make one hop (i.e., $\beta=1$), the problem reduces to the baseline model, and \Cref{thm:two-populations}'s guarantee becomes the same as that of \Cref{thm:biased-weighted-maj-vote-mistake-bound} with the same set of parameters.
    For values of $\beta$ between 0 and 1, the mistake bound smoothly interpolates the guarantees of the two extreme cases.
\end{remark}
    \begin{proof}[Proof of~\Cref{thm:two-populations}]
        
        We show that whenever a mistake is made, we can reduce the total weight of experts ($W_t$) by a constant fraction in expectation.

     First, consider the case of a false positive. 
     Since $h_t(v_t)=+1$, the total weight of experts that predict positive on $v_t$ is at least $\theta W_t$ %
     ; and {the weight of each of them gets reduced} by a factor of $\lambda$. 
     Let $\cF_t$ be the $\sigma$-algebra generated by the random variables up to time $t$, then
     we have 
     \begin{align}
         \E[W_{t+1}\ |\ \mathcal{F}_{t-1},\text{ false positive}]\le W_t (1-\lambda \theta).
         \label{eq:cut-false-positive}
     \end{align}

     Next, consider the case of a false negative. Since $h_t(v_t)=-1$, we know that the agent did not move, i.e., $v_t=u_t$. The algorithm updates as follows: if $c_t=B$, it reduces the weight of experts who predict   {negative on all the nodes in the one-hop neighborhood of $v_t$, i.e., $N[v_t]$.} if $c_t=A$, then it reduces the weight of experts who predict  {negative on all the nodes in the two-hop neighborhood of $v_t$, i.e., $N^2[v_t]$}. We claim that:
     
    \begin{align*}
        &\frac{\Pr(c_t=B\ |\ \cF_{t-1},\text{false negative})}{\beta}\ge \frac{\Pr(c_t=A\ |\ \cF_{t-1},\text{false negative})}{1-\beta}\\ \Rightarrow\ &\Pr(c_t=B\ |\ \cF_{t-1},\text{false negative})\ge\beta.
    \end{align*}
    To see this, we can use the Bayes law to calculate the conditional probability of group assignments: for {$X\in \{A,B\}$}, we have
    \begin{align}
        \Pr(c_t=X\ |\ \cF_{t-1},\text{false negative})=\frac{\Pr(\text{false negative}\ |\ \cF_{t-1},c_t=X)\cdot\Pr(c_t=X\ |\ \cF_{t-1})}{\Pr(\text{false negative}\ |\ \cF_{t-1})}.
    \end{align}
    Since the group membership $c_t$ is independently realized after the adversary chooses $(u_t,y_t)$, we have
    \begin{align*}
        &\frac{\Pr(c_t=B\ |\ \cF_{t-1},\text{false negative})}{\beta}-\frac{\Pr(c_t=A\ |\ \cF_{t-1},\text{false negative})}{1-\beta}\\
        =&\frac{1}{\Pr(\text{false negative}\ |\ \cF_{t-1})}\Big(\Pr(\text{false negative}\ |\ \cF_{t-1},c_t=B)-\Pr(\text{false negative}\ |\ \cF_{t-1},c_t=A)\Big)\ge0,
    \end{align*}
    where the last step is because agents of population $A$ have more manipulation power, so under every possible classifier, group $A$ is able to get classified as positive whenever group $B$ is; therefore, group $A$ agents are less likely to become false negative. We have thus established the claim.

    Now we turn to the total weight that is reduced in this scenario. If $c_t=A$, then there are at most $(\Delta^2+1)$ nodes in the two-hop neighborhood, in which all of them are predicted negative. Therefore, the total weight of experts who predict negative on all of them is at least $W_t(1-\theta(\Delta^2+1))_+$. On the other hand, if $c_t=B$, then the total weight of experts who predict negative on the one-hop neighborhood is at least $W_t(1-\theta(\Delta+1))_+$.
    Putting the two cases together and conditioning on the false negative, the total weight that can be reduced is at least
    \begin{align}
    &\Pr(c_t=B\ |\ \text{false negative},\cF_{t-1})\cdot {(1-(\Delta+1)\theta)_+}
    \nonumber\\&\qquad\qquad
        +\Pr(c_t=A\ |\ \text{false negative},\cF_{t-1})\cdot (1-(\Delta^2+1)\theta)_+\nonumber\\
        \ge& \max\left\{\beta(1-(\Delta+1)\theta)_+,(1-(\Delta^2+1)\theta)_+\right\},
        \label{eq:tttmp}
    \end{align}
    where the first term in \eqref{eq:tttmp} is due to the claim we just established, and the second term follows from $(1-(\Delta+1)\theta)_+\ge(1-(\Delta^2+1)\theta)_+$ together with
    $\Pr(c_t=B\ |\ \text{false negative},\cF_{t-1})+\Pr(c_t=A\ |\ \text{false negative},\cF_{t-1})=1$.
    From \Cref{eq:tttmp}, we obtain
    \begin{align}
        \E[W_{t+1}\ |\ \cF_{t-1},\text{ false negative}]\le W_t\left(1-\lambda\cdot\max\left\{\beta(1-(\Delta+1)\theta)_+,(1-(\Delta^2+1)\theta)_+\right\}\right).\label{eq:cut-false-negative}
    \end{align}
Finally, we optimize the threshold $\theta$ to equalize the decrease in the case of false positive (\Cref{eq:cut-false-positive}) and false negative (\Cref{eq:cut-false-negative}). As a result, the optimal $\theta$ is obtained by solving the following equation:
\begin{align}
    \underbrace{\theta}_{f(\theta)}=\max\left\{\underbrace{\beta(1-(\Delta+1)\theta)}_{f_1(\theta)},\,\underbrace{1-(\Delta^2+1)\theta}_{f_2(\theta)},\,0\right\}.\nonumber
\end{align}
Since $f,f_1$, and $f_2$ are all linear functions where $f_1,f_2$ have a negative slope and $f$ has a positive slope, the intersection between $f$ and $\max\{f_1,f_2\}$ coincides with the maximum value between the intersection of $\{f,f_1\}$ and the intersection of $\{f,f_2\}$. Moreover, $\theta=0$ is not a valid solution because the other two intersections have strictly positive values.
Thus we obtain
\begin{align*}
    \theta\triangleq\max\left\{\frac{1}{\Delta+1+\frac{1}{\beta}},\ \frac{1}{\Delta^2+2}\right\}.
\end{align*}
Correspondingly, on each mistake, the optimal amount of decrease in the total weight is
\begin{align*}
    \E\left[\left.\frac{W_{t+1}}{W_t}\ \right|\ \cF_{t-1},\text{ mistake}\right]\le \min\left\{1-\frac{\lambda}{\Delta+1+\frac{1}{\beta}},\ 1-\frac{\lambda}{\Delta^2+2}\right\}.
\end{align*}
By Jensen's inequality, we further obtain that if a mistake is made at time $t$, then
\begin{align}
    \E\left[\ln \left.\frac{W_{t+1}}{W_t}\ \right|\ \cF_{t-1},\text{ mistake}\right]\le&\ln\E\left[\left.\frac{W_{t+1}}{W_t}\ \right|\ \cF_{t-1},\text{ mistake}\right]\nonumber\\
    \le &\ln\left(\min\left\{1-\frac{\lambda}{\Delta+1+\frac{1}{\beta}},\ 1-\frac{\lambda}{\Delta^2+2}\right\}\right).\label{eq:decrease:mistake}
\end{align}
The last step is to telescope \Cref{eq:decrease:mistake} over all mistakes.
Note that the algorithm only penalizes the experts that make mistakes, so the same argument as \Cref{thm:biased-weighted-maj-vote-mistake-bound} implies that $W_T\ge\gamma^{\OPT}$. Thus we have
\begin{align*}
\E[\ln(\gamma^{\OPT})-\ln|\cH|] \le \E[\ln W_T-\ln|\cH|]
    \le    \E[\mistake(T)]\cdot\ln\left(\min\left\{1-\frac{\lambda}{\Delta+1+\frac{1}{\beta}},\ 1-\frac{\lambda}{\Delta^2+2}\right\}\right).
\end{align*}
 
 {Rearranging} the above inequality, setting $\lambda=1/e$ {and using $\ln(1-x)\leq -x$} gives us an expected mistake bound of 
 \begin{align*}
     \E[\mistake(T)]\le e\cdot\min\left\{\Delta+1+\frac{1}{\beta},\ \Delta^2+2\right\}\left(\ln|\cH|+\E[\OPT]\right).
 \end{align*}
 This completes the proof.
    \end{proof}

\end{document}